\documentclass[11pt]{article}

\usepackage[a4paper,margin=3cm]{geometry}
\usepackage[T1]{fontenc}
\usepackage[utf8]{inputenc}
\usepackage{amsmath,amssymb,amsthm}
\usepackage{mathtools}

\usepackage{dsfont}
\usepackage{microtype}
\usepackage{enumerate}
\usepackage[dvipsnames]{xcolor}
\usepackage{float}
\usepackage[colorlinks=true,citecolor=blue]{hyperref}
\usepackage{bm}
\usepackage{csquotes}
\MakeOuterQuote{"}
\usepackage[affil-it]{authblk}

\usepackage{orcidlink}
\usepackage{algorithm}
\usepackage{algpseudocode}

\usepackage{mathrsfs}
\usepackage{booktabs}
\usepackage{siunitx}
\usepackage[authoryear]{natbib}
\setcitestyle{authoryear,round}
\newcommand{\R}{\mathbb{R}}
\newcommand{\N}{\mathbb{N}}
\newcommand{\suppX}{\mathbb{X}}
\newcommand{\paramspace}{\bm{\Theta}}
\newcommand{\reparamspace}{\bm{\Upsilon}}
\renewcommand{\P}{\mathsf{P}}
\newcommand{\E}{\mathsf{E}}
\newcommand{\Var}{\mathsf{Var}}
\newcommand{\lp}{\left(}
\newcommand{\rp}{\right)}
\newcommand{\lc}{\left[}
\newcommand{\rc}{\right]}
\newcommand{\lacc}{\left\{}
\newcommand{\racc}{\right\}}
\newcommand{\labs}{\left|}
\newcommand{\rabs}{\right|}
\newcommand{\data}{\mathcal{D}_n}

\newcommand{\I}{\mathbb{I}}

\newcommand{\cG}{\mathcal{G}}
\newcommand{\cH}{\mathcal{H}}
\DeclareMathOperator{\lin}{lin}
\DeclareMathOperator{\conv}{conv}

\newcommand{\RCData}{\mathfrak{R}_{n}}
\newcommand{\eps}{\varepsilon}
\newcommand{\eqlaw}{\stackrel{\mathcal{L}}{=}}
\newcommand{\cZ}{\mathbb{Z}}
\newcommand{\cA}{\mathcal{A}}

\DeclareMathOperator{\clamp}{clamp}
\newcommand{\outputspace}{\lin_T^{c}(\cH)|_{\reparamspace}}
\newcommand{\event}{\mathcal{E}}
\newcommand{\borel}{\mathcal{B}}
\newcommand{\hF}{\widehat{\bm{F}}}
\newcommand{\diff}{\mathrm{d}}
\newcommand{\ee}{\mathrm{e}}
\newcommand{\X}{\bm{X}}

\newcommand{\x}{\bm{x}}
\newcommand{\bF}{\overline{F}}
\newcommand{\bH}{\overline{H}}
\newcommand{\RV}{\mathrm{RV}}
\newcommand{\btheta}{\bm{\theta}}
\newcommand{\hR}{\widehat{R}}
\newcommand{\hu}{\widehat{u}}

\newcommand{\F}{\bm{F}}

\newcommand{\G}{\bm{G}}
\newcommand{\h}{\bm{h}}
\newcommand{\clip}{c}
\newcommand{\tr}{\intercal}
\newcommand{\VC}{\mathcal{V}}
\DeclareMathOperator{\Be}{Be}
\DeclareMathOperator{\Bin}{Bin}
\newcommand{\cst}{\mathscr{C}}

\newcommand{\Pdata}{\widehat{P}_{n,(\X,Y)}}
\newcommand{\Ptrue}{P_{(\X,Y)}}
\newcommand{\spaceZ}{\mathbb{Z}}
\newcommand{\dst}{\mathscr{D}}
\newcommand{\kst}{\mathscr{K}}
\newcommand{\ist}{\mathscr{I}}
\renewcommand{\eta}{\tau}
\newcommand{\bphi}{\bm\phi}

\theoremstyle{plain}
\newtheorem{theorem}{Theorem}

\newtheorem{lemma}{Lemma}
\newtheorem{assumption}{Assumption}
\theoremstyle{remark}

\title{Gradient boosting for extremes: \\ sampling theory and application to insurance}

\author[1]{Stéphane Lhaut \orcidlink{0000-0003-2449-6218} \thanks{\href{mailto:stephane.lhaut@ensae.fr}{stephane.lhaut@ensae.fr}, corresponding author}}
\author[1]{Olivier Lopez \orcidlink{0000-0002-6158-4113} \thanks{\href{mailto:olivier.lopez@ensae.fr}{olivier.lopez@ensae.fr}}}
\affil[1]{CREST, CNRS, Ecole polytechnique, Groupe ENSAE-ENSAI, ENSAE Paris, Institut Polytechnique de Paris, Palaiseau, France}

\date{\today}

\allowdisplaybreaks

\begin{document}
	
\maketitle

\begin{abstract}
	We develop a statistical learning theory for gradient boosting applied to the estimation of covariate-dependent Generalized Pareto (GP) distributions in the context of Peaks-over-Threshold modeling. After an orthogonal reparametrization of the GP likelihood that diagonalizes its Fisher information matrix, we cast the estimation problem within the Empirical Risk Minimization (ERM) framework and derive non-asymptotic error bounds for the boosting estimator. Our analysis accounts for three distinct sources of error in the process: statistical fluctuations, the approximation bias inherent to the asymptotic nature of the GP model---controlled under second-order regular variation---and the approximation error associated with the finite number of boosting iterates, making explicit the resulting bias–variance trade-off. We illustrate the practical benefits of the reparametrization through simulations, showing that it significantly reduces gradient correlation during training and improves convergence stability. The methodology is applied to a medical malpractice insurance dataset from the Texas Department of Insurance, comprising over 18 000 closed claims. The gradient boosting approach yields a good fit for the tail of settlement cost distributions and reveals that the number of days to settlement is the dominant predictor of tail heaviness, consistent with earlier findings in the reserving literature.
\end{abstract}

\section{Introduction}
\label{sec:intro}

Heavy-tail distributions appear naturally in many domains of statistical applications, such as insurance, finance, or natural hazards modeling. Evidence of power law behavior can be found in cyber risk when it comes to evaluating the magnitude of a data breach \citep{maillart2010heavy, farkas2021cyber}, market returns \citep{borak2011models}, and extreme precipitation \citep{gimeno2022extreme}. Numerous examples can be found in \citep{beirlant2006statistics, embrechts2013modelling}.

Extreme Value Theory (EVT) is the field of statistics devoted to the analysis of the tail behavior of probability distributions \citep{resnick1987extreme, resnick2007heavy}. The Generalized Pareto (GP) distribution plays a key role in this theory. Indeed, for many probability distributions, including the ones with power law behavior, it can be used as an (asymptotic) model to approximate their upper tail, in view of the Pickands–Balkema–de Haan theorem \citep{pickands1975statistical, balkema1974residual}.

In the heavy-tail case, such probability distributions are those whose survival functions are power laws, up to a slowly varying function \citep[Section 2.3]{beirlant2006statistics}. The tail index appearing in these power laws is the central quantity for describing the heaviness of the tail and the risk associated with the quantity of interest. Many procedures exist to estimate this parameter from data, including the Hill estimator \citep{hill1975simple} or the maximum likelihood method \citep[Chapter 5]{beirlant2006statistics}. Here we are interested in estimating this tail index conditionally on input information, for which many approaches have been considered in the literature.

\paragraph*{Regression for tail risk analysis.}
Standard regression approaches in insurance are mainly based on the Exponential Dispersion family, the Generalized Linear Model, and their extensions \citep{denuit2019effective}. The main goal of these models is to capture the conditional mean behavior of the risk of interest. However, in the presence of heavy tails, these constructions may lack the flexibility to correctly capture the tail behavior of the loss of interest. As such, fitting a specific tail model based on the GP distribution \citep{davison1990models} may help improve risk quantification and reserving.

Many approaches have been considered in the literature, including local polynomials \citep{beirlant2004local}, generalized additive models \citep{chavez2005generalized, chavez2016gam}, regression trees \citep{farkas2024gptree}, random forests \citep{gnecco2024ERF}, gradient boosting trees \citep{velthoen2023gradient}, and neural networks \citep{pasche2024neural}. Each method has its own merits and drawbacks depending on the application. On the one hand, algorithmically lighter methods like generalized additive models or trees have the benefit of being easier to interpret but may lack the flexibility to correctly capture complex relationships with a high number of covariates. On the other hand, "machine learning" approaches like gradient boosting or neural-based methods are flexible but lack interpretability. These issues are discussed specifically for the insurance context in \citet{maillart2021}.

In this work, we focus on gradient boosting methods \citep{mason1999boosting, mason2000functional, friedman2001greedy} as we aim for a flexible approach able to deal with a large number of covariates of mixed types. In the context of extremes, such algorithms were applied for the first time in \citet{velthoen2023gradient}, with trees as base learners. However, no statistical or algorithmic theory was developed. Here, after an orthogonal reparameterization \citep{cox1987parameter} of the GP likelihood, we discuss the statistical guarantees of boosting approaches within the Empirical Risk Minimization (ERM) paradigm \citep{vapnik1991principles}. Note that this paradigm focuses solely on the statistical part of the error and is not concerned with any algorithmic consideration. For this theory, we do not need to use trees as basis functions for the boosting procedure, but merely to have control over the complexity of the function class. In practice, we rely on our re-parametrized implementation of the algorithm developed in \citet{velthoen2023gradient} using trees as base learners.

\paragraph*{Outline.}
Section \ref{sec:framework} introduces the Peaks-over-Threshold (POT) method used to fit the GP distribution parameters, together with the gradient boosting approach for fitting these parameters as functions of inputs. Section \ref{sec:theory} describes the statistical guarantees we obtain on ERM solutions of boosting algorithms in this context, taking into account the different sources of error: statistical fluctuations, bias due to the asymptotic nature of the GP model, and approximation error. Section \ref{sec:numerical} presents numerical results: Section \ref{subsec:simus} illustrates the advantages of the reparameterization, and Section \ref{subsec:application} contains the application to medical malpractice data in Texas. Section \ref{sec:ccl} concludes the paper. Appendix \ref{app:notation} provides the notation we use throughout the paper, while Appendices \ref{app:auxiliary} and \ref{app:proofs} contain auxiliary results and the proofs of the main results. Appendix \ref{app:figures} contains additional figures.

\section{Gradient boosting for extremes}
\label{sec:framework}

Our main interest lies in describing possibly very larges values of an output $Y$ with values in $\R_+$ based on an input $\X$ with values in $\suppX \subset \R^d$ for some $d \in \N_0$. We focus on heavy-tailed responses $Y$ with a survival function of power-law type at infinity.

\begin{assumption}[Regularly varying tail]
\label{ass:tail-of-Y}
For any $\x \in \suppX$, there exists $\ell(\cdot \mid \x) \in \RV_0$ where $\RV_0$ is the set of slow-varying functions (see precise definition in Appendix \ref{app:notation}) and $\xi_0 : \suppX \to \R_+$ such that 
\[
	\bF_{Y \mid \X = \x}(y) = \ell(y \mid \x) y^{-1/\xi_0(\x)}, \qquad \forall y > 0.
\]
\end{assumption}

Assumption \ref{ass:tail-of-Y} is equivalent to $\bF_{Y \mid \X = \x} \in \RV_{-1/\xi_0(\x)}$ and implies that when $y \to \infty$, the conditional survival probability of the random $Y$ exceeding $y$ given $\X = \x$ behaves like a power function. If the tail index $\xi_0(\x) > 0$ is large, the tail decay of $Y$ given $\X = \x$ is very slow and very heavy realizations of $Y$ may occur. For example, the case $\xi_0(\x) > 1$ is of special interest as the expectation $\E[Y \mid \X = \x]$ is not even finite for such values.

Assumption \ref{ass:tail-of-Y} is equivalent to the attraction of the (conditional) excess of $Y$ given $\X = \x$ above a large threshold $u(\x)$ to a GP distribution as $u(\x) \to \infty$. This reformulation will drive our estimation procedure. See Theorem 1.2.5 in \citet{deHaanF2006EVT} for the precise statement.

\begin{theorem}[Pickands--Balkema--de Haan \citep{pickands1975statistical}]
\label{thm:P-B-dH}
Under Assumption \ref{ass:tail-of-Y}, there exists a function $(u, \x) \in \R_+ \times \suppX \mapsto \sigma_0(u,\x)$ such that for any $\x \in \suppX$,
\begin{equation}
\label{eq:P-B-dH}
    \lim_{u(\x) \to \infty} \sup_{z>0} 
    \labs
        \overline{F}_{u(\x)} (z \mid \x) - \overline{H} (z ; (\sigma_0(u(\x), \x), \xi_0(\x) ))
    \rabs = 0
\end{equation}
where for $z,u > 0$ and $\x \in \suppX$,
\[
	\bF_u(z \mid \x) = \P \lp Y - u > z \mid Y > u, \X = \x \rp
\]
and for $\btheta = (\sigma, \xi) \in \paramspace \subseteq \R_+^2$,
\begin{equation}
\label{eq:GPD-survival}
	\bH \lp z ; \btheta \rp = \lp 1 + \xi \frac{z}{\sigma} \rp^{-1/\xi}.
\end{equation}
Furthermore, if such a function $\sigma_0$ exists, then \eqref{eq:P-B-dH} also holds with $\sigma_0(u,\x) = u \cdot \xi(\x)$.
\end{theorem}

In view of Theorem \ref{thm:P-B-dH}, estimating the size of high values of $Y$ given $\X=\x$ amounts to estimating the bivariate function $\btheta_0 = (\sigma_0, \xi_0)$.

\paragraph*{ERM with GP negative log-likelihood.}

We now cast our problem as a risk minimization procedure in function space whose solution will be approached by en ERM procedure.
As already pointed out by several authors \citep{chavez2005generalized, pasche2024neural}, one crucial point to make the theory work and to improve the results in the application is to reparametrize the GP distribution in order to obtain a diagonal Fisher information matrix using the procedure in \citet{cox1987parameter}. 
This leads to the reparametrization
\begin{equation}
\label{eq:ortho-reparam}
    \btheta = (\sigma, \xi) \in \paramspace
    \to 
    \bphi = (\nu, \xi) \in \reparamspace
    \qquad \text{with} \qquad
    \nu = \sigma (1 + \xi)
\end{equation}
The associated Fisher information is the diagonal matrix \citep{moins2023reparam}
\begin{equation}
\label{eq:fisher-true}
    I(\nu,\xi) =
    \begin{bmatrix}
        \frac{1}{\nu^2(1+2\xi)} & 0 \\
        0 & \frac{1}{(1+\xi)^2}
    \end{bmatrix}.
\end{equation}

The negative log-likelihood associated with an observation $z>0$ from a GP distribution with parameter $\bphi = (\nu,\xi) \in \reparamspace$ is given by
\[
	L(\bphi, z) = \lp 1 + \frac{1}{\xi} \rp \log \lp 1 + \frac{\xi(1+\xi)}{\nu} z \rp + \log(\nu) - \log(1+\xi). 
\]
Given a map $\F : \suppX \to \reparamspace$, the population risk that we aim to minimize is thus given by
\begin{equation}
\label{eq:true-risk}
	R(\F) = \E \lc L(\F(\X), Z) \rc,
\end{equation}
where the expectation is taken over the joint distribution of $(\X,Z)$ with $Z \mid \X = \x$ following the limiting GP model with unknown parameter $\btheta_0(\x)$.

Of course, this risk is never observed in practice and minimization has to be performed based on data. Furthermore, no observable data directly comes from the limiting GP model and estimation has to be performed at a pre-asymptotic level. Both approximation will lead to errors and we will show how to control them in Section \ref{sec:theory}.

We assume that we are given a collection 
\[
	\data = \lacc (\X_i, Y_i): \; i \in [n] \racc 
\]
of i.i.d. copies of the random pair $(\X,Y)$ with $Y$ satisfying Assumption \ref{ass:tail-of-Y}. 
For $\F : \suppX \to \reparamspace$, based on Theorem \ref{thm:P-B-dH}, it is natural to estimate the unknown asymptotic population risk in \eqref{eq:true-risk} by the risk
\begin{equation}
\label{eq:empirical-risk-true-u}
	\hR_n^{u}(\F) = \frac{1}{k_n} \sum_{i=1}^{n} L \lp \F(\X_i), Y_i - u(\X_i ; \tfrac{k_n}{n}) \rp \I \lacc  Y_i > u(\X_i ; \tfrac{k_n}{n}) \racc,
\end{equation}
where $(k_n)_{n \in \N}$ is an intermediate sequence such that $k_n/n \to 0$ and $k_n \to \infty$ as $n \to \infty$ and
\[
    u(\x, p) = F_{Y \mid \X = \x}^{-1}(1-p), \qquad
    p \in (0,1), \x \in \suppX
\]
is the $p$-th upper quantile of $Y \mid \X = \x$.
The risk in \eqref{eq:empirical-risk-true-u} corresponds to the empirical risk based on the excesses 
\[
	Y_i - u(\X_i ; \tfrac{k_n}{n}) \mid Y_i > u(\X_i ; \tfrac{k_n}{n}), \qquad i \in [n],
\]
associated with $\data$. 

This risk is also unobserved in practice because the intermediate quantile $u(\x; \tfrac{k_n}{n})$ of the conditional law $\P_{Y \mid \X = \x}$ is unknown. 
We will assume that we are given some estimator $\hu(\x; \tfrac{k_n}{n})$ of this quantity, which may be obtained by standard quantile regression techniques, and we show in Section \ref{sec:theory} how the quality of this estimator influences the final estimate. 
This leads to the following risk that we will minimize in practice
\begin{equation}
	\label{eq:empirical-risk-empirical-u}
	\hR_n^{\hu}(\F) = \frac{1}{k_n} \sum_{i=1}^{n} L \lp \F(\X_i), Y_i - \hu(\X_i ; \tfrac{k_n}{n}) \rp \I \lacc  Y_i > \hu(\X_i ; \tfrac{k_n}{n}) \racc.
\end{equation}

The risks associated with our problem can't be minimized over the full space of functions $\F : \suppX \to \reparamspace$ and we need to restrict the class of functions we will be looking at. The gradient boosting approach corresponds to a smart way to minimize the risk over linear combinations of ``simple'' functions, leading to a powerful estimator in practice.

\paragraph*{Gradient boosting.}

The gradient boosting algorithm considers a "hypothesis class" $\cH$ of functions $\h : \suppX \to \R^2$, also called the "base learners" or "weak learners", assumed to be negation closed ($\h \in \cH \iff -\h \in \cH$) as in \citet{biau2021optimization}, and minimize the (empirical) risk over $\lin(\cH)$ the space of functions consisting of all the (positive) linear combinations of functions in $\cH$, i.e.,
\[
	\lin(\cH) = \bigcup_{J \geq 1} \lin_J(\cH),
\]
where for $J \in \N_0$
\[
	\lin_J(\cH)
	=    
	\left\{ 
		\F(\x)
		= \sum_{j=1}^J \alpha_j \h_j(\x) : (\alpha_1, \ldots, \alpha_J) \in \R^J_+, \quad \h_1, \ldots, \h_J \in \cH 
	\right\}.
\]

To perform the minimization in practice, an algorithm is required. We will consider the following procedure (Algorithm \ref{alg:GB}) which is the one proposed in \citet{velthoen2023gradient}, based on \citet{friedman2001greedy}, incorporating gradient clipping as advised by \citet{velthoen2023gradient} and adding a "clamp" step at the end to make sure that the output lies in the parameter space $\reparamspace$.
For every function $\F \in \lin(\cH)$ of the form $\F(\x) = (\nu(\x), \xi(\x))$ for $\x \in \suppX$, we denote by $\clamp(\F,\reparamspace)$ the function defined for every $\x \in \suppX$ by
\[
	\clamp(\F,\reparamspace)(\x)
	= \lp
	\min \{ \max(\nu(\x),\nu_L), \nu_U \},
	\min \{ \max(\xi(\x),\xi_L), \xi_U \}
	\rp
	\in \reparamspace,
\]
where $\nu_L, \nu_U = \nu_U(n), \xi_L, \xi_U > 0$ are introduced in Assumption \ref{ass:compact} below.

\begin{algorithm}
	\caption{Gradient Boosting algorithm}
	\label{alg:GB}
	\begin{algorithmic}
		\Require Learning rates $(\lambda_t)_{t \in \N_0} > 0$, clipping parameter $\clip > 0$ and parameter space $\reparamspace$
		\State $\F_0 \in \cH$
		\For{$t = 1,2,\ldots,T$}
		\State $\h_t \gets \operatorname{argmin}_{\h \in \cH} \sum_{i=1}^n \left\| -\nabla_{\bphi} L \lp \F_{t-1}(\X_i), Y_i - \hu(\X_i ; \tfrac{k_n}{n}) \rp \I \lacc  Y_i > \hu(\X_i ; \tfrac{k_n}{n}) \racc - \h(\X_i) \right\|_2^2$
		\State $\F_t = \F_{t-1} + \min\{\lambda_t, \clip\} \, \h_t$
		\EndFor
        \State $\F: \x \in \suppX \mapsto \F(\x) = (\nu(\x), \xi(\x)) = \clamp(\F_T, \reparamspace)(\x)$
		\Ensure $\btheta: \x \in \suppX \mapsto \btheta(\x) = \lp \frac{\nu(\x)}{1+\xi(\x)}, \xi(\x) \rp$
	\end{algorithmic}
\end{algorithm}

At each step, the algorithm searches for the best direction in the weak learner space, in the sense that it looks for the one which is best at predicting the negative gradient of the loss. It then updates the output function by a weighted version of the resulting weak learner.

The convergence of Algorithm \ref{alg:GB} when $T \to \infty$ is studied in Theorem 2 of \citet{biau2021optimization}, provided that the learning rates $\lambda_t > 0$ are chosen small enough. As pointed by \citet{velthoen2023gradient}, the loss function $L$ is not convex in our case as the negative log-likelihood associated with the GP model is not globally convex in general so that Assumption 2 in \citet{biau2021optimization} is not verified and we can't hope to reach a global minimum, even for the reparametrized loss. 
Note however that the algorithm proposed by \citet{mason2000functional} is also studied in \citet{biau2021optimization} and its convergence as $T \to \infty$ can be obtained under lighter conditions. See Section 5 of \citet{mason2000functional} for convergence results without convexity of the loss.
To obtain statistical guarantees on the ERM solution of the empirical risk in \eqref{eq:empirical-risk-empirical-u}, it will be sufficient to be a local minima inside the class and this is achieved by Algorithm \ref{alg:GB} as $T \to \infty$.
Of course, in practice we stop the algorithm after a finite number $T \in \N_0$ of iterations and this parameter will play a role as trade off between sample complexity and accuracy of the final regression function as explained in the next sections.

\section{Sampling theory}
\label{sec:theory}

The error made in estimating $\btheta_0$ can be decomposed into three parts. Indeed, for fixed $\F: \suppX \to \reparamspace$, we can consider separately two terms while dealing with the risk:
\[
    \labs \hR_n^{\hu}(\F) - R(\F) \rabs
    \leq
    \labs \hR_n^{\hu}(\F) - R^u(\F) \rabs
    +
    \labs R^u(\F) - R(\F) \rabs,
\]
where
\[
		R^{u}(\F) 
		= \E \lc \hR_n^{u}(\F) \rc
		= \frac{n}{k_n} \E \lc  L \lp \F(\X), Y - u(\X ; \tfrac{k_n}{n}) \rp \I \lacc  Y > u(\X; \tfrac{k_n}{n}) \racc \rc.
\]

The first error associated with the first term may be called a \emph{stochastic error} and corresponds to the fact that we only have access to a finite amount of data. Note that this error takes into account the fact that the true intermediate quantile $u(\cdot; \tfrac{k_n}{n})$ is unknown and is estimated using the data. This error decreases as $n \to \infty$.

The second error associated with the second term may be called a \emph{bias error} and corresponds to the fact that $R^u(\F)$ is only an approximation of the true asymptotic risk $R(\F)$ arising from the GP model. This error decreases as $u \to \infty$.

The third and last part of the error may be called an \emph{approximation error} and is associated with the fact that the true parameter function is not a finite linear combinations of elements of the hypothesis class $\cH$. This error decreases as $T \to \infty$.

Auxiliary results are required to study these different error terms. They are gathered in Appendix \ref{app:auxiliary}. Some of these results are directly taken from the literature and are stated without proof, while some are proven in the section. The proofs are deferred to Appendix \ref{app:proofs}. The main assumptions underlying our results are stated below with discussion.

\subsection{Assumptions and discussion}

As it is classically needed in this kind of analysis, we will assume compactness of $\reparamspace$.
\begin{assumption}[Compact parameter space]
\label{ass:compact}
    We assume that the parameter space $\reparamspace = \reparamspace(n)$ takes the form
    \[
        \reparamspace
        = [\nu_L, \nu_U] \times [\xi_L, \xi_U]
    \]
    for some $0 < \nu_L < \nu_U = \nu_U(n)$ and $0 <  \xi_L < \xi_U$.
\end{assumption}
Note that the upper bound on the (modified) scaled parameter $\nu$ may depend on the sample size $n$ since, as explained in Theorem \ref{thm:P-B-dH}, the scale parameter may always be taken proportional to the threshold $u$.

\begin{assumption}
\label{ass:k_n}
$k_n = n^\alpha$ for some $\alpha \in (0,1)$.
\end{assumption}

The following condition is a quantitative version of Assumption 2 in \citet{gnecco2024ERF}. This assumption is weaker than requiring that the estimated quantiles converge to the true counterparts since only the ratio needs to be close to one. A possible choice for such a uniformly consistent method is given in \citet{wang2013estimation}.

\begin{assumption}[Intermediate quantile estimation]
	\label{ass:quantile-estimation}
	For any $\delta \in (0,1)$, there exists an event with probability at least $1-\delta$ on which
	\[
		\sup_{\x \in \suppX} \labs \frac{\hu(\x;k_n/n)}{u(\x;k_n/n)} - 1 \rabs \leq a_n(\delta)
	\]
	for some $a_n(\delta) > 0$ such that $a_n(\delta) \to 0$ as $n \to \infty$.
\end{assumption}

The next assumption ensures that the hypothesis class $\cH$ is not ``too large'' and admits a reasonable bound. This is not the only possible assumption but will be sufficient for our purposes as we will be using small regression trees in practice and the condition is verified with these functions, as explained in the beginning of Section \ref{sec:numerical}.

\begin{assumption}[Classes of hypothesis]
	\label{ass:hypothesis}
	For $j \in [2]$, the class
	\[
		\cH_j
		=
		\lacc
		h^{(j)} : \suppX \to \R : \h \in \cH
		\racc
	\]
	containing the $j$-th margin of functions in $\cH$ is pointwise measurable \citep[Example 2.3.4]{vdVW1996weak}, of finite VC dimension $\VC_j$ and admits an envelope function $H_j : \suppX \to \R$ such that
	\[
		\|H_j\|_{L_2(\suppX, P_{\X})}  < \infty,
	\]
	and
	\[
		\sup_{\h \in \cH} \sup_{\x \in \suppX} h^{(j)}(\x)
		\leq 
			\sup_{\x \in \suppX} H_j(\x)
		\leq
			\frac{1}{4} \sqrt{\frac{n\|H_j\|_{L_2(P_X)}}{\log \lp \mathscr{K} \VC_j (16\mathrm{e})^{\VC_j}\rp}},
	\]
	where $ \mathscr{K}>0$ is a universal constant not depending on the function class.
\end{assumption}
Since $\cH_j$ admits a finite VC dimension $\VC_j$ with $L_2$ envelope function $H_j$, Theorem 2.6.7 in \citet{vdVW1996weak} implies that its covering number $\mathcal{N}(\eps, \cH_j)$ as defined in equation (2.2) of \citet{einmahl2005uniform} satisfies
\[
	\mathcal{N}(\eps, \cH_j)
	\leq
	\mathscr{K} \VC_j (16\ee)^{\VC_j} \eps^{-2\VC_j}, \qquad \forall \eps \in (0,1).
\]
This is needed in order to apply Proposition 1 in \citet{einmahl2005uniform}. The uniform bound is related to assumption (iv) in their result, with $\nu = 2\VC_j$, $\beta^2 = \sigma^2 = \|H_j\|_{L_2(\suppX, P_{\X})}$ and $C_1 = \lp \mathscr{K} \VC_j (16\ee)^{\VC_j} \rp^{1/\nu}$.

The following conditions is similar to Assumption 2 in \citet{farkas2024gptree}.

\begin{assumption}[Fisher information]
\label{ass:fisher}
There exists $\ist > 0$ such that for $\G \in \lin(\cH)$,
\[
    \inf_{\bphi \in \reparamspace}
    \frac{n}{k_n}
    \left\|
    \E \lc \nabla_{\bphi}^2 L \lp \bphi, Y - u(\X;\tfrac{k_n}{n}) \rp^\tr \G(\X) \I \lacc Y > u(\X;\tfrac{k_n}{n}) \racc \rc
    \right\|_\infty
    \geq \ist \E \lc \|\G(\X)\|_\infty \rc.
\]
\end{assumption}

The orthogonal reparametrization that we introduced in \eqref{eq:ortho-reparam} is key in order to make Assumption \ref{ass:fisher} possible. Indeed, take $\bphi = (\nu,\xi) \in \reparamspace$ and $\G \in \lin(\cH)$, then
\begin{multline*}
    \frac{n}{k_n}
    \E \lc \nabla_{\bphi}^2 L \lp \bphi, Y - u(\X;\tfrac{k_n}{n}) \rp^\tr \G(\X) \I \lacc Y > u(\X;\tfrac{k_n}{n}) \racc \rc \\
    =
    \int_{\suppX} \E_{Y \mid \X = \x} \lc \nabla_{\bphi}^2 L \lp \bphi, Y - u(\x;\tfrac{k_n}{n}) \rp \mid Y > u(\x;\tfrac{k_n}{n})\rc^\tr \G(\x) \, \diff P_{\X}(\x)
\end{multline*}
Using Assumption \ref{ass:tail-of-Y}, the inner expectation converges to $I(\nu_0(\x), \xi_0(\x))$ where the Fisher information $I$ has been defined in \eqref{eq:fisher-true}. Both non-zero elements of this matrix are bounded away from zero by Assumption \ref{ass:compact}, making Assumption \ref{ass:fisher} realistic in our setting.

The next condition comes from point 11 of Proposition B.1.9 in \cite{deHaanF2006EVT}.

\begin{assumption}[Slowly varying function]
\label{ass:sv-function}
	For $\x \in \suppX$, the map $y > 0 \mapsto \ell'(y \mid \x)$ can be rewritten as
    \[
        \ell'(y \mid \x) = \ell(y \mid \x) \cdot y^{-1} \cdot r(y \mid \x),
    \]
    where $y > 0 \mapsto \labs r(y \mid \x) \rabs$ is a decreasing function such that $\sup_{\x \in \suppX} \labs r(y \mid \x) \rabs \searrow 0$ as $y \to \infty$.
\end{assumption}

\subsection{Stochastic error}
\label{subsec:stochastic}

Recall that the output of Algorithm \ref{alg:GB} is an element of $\outputspace$ where
\[
	\outputspace 
	=
	\lacc
		\clamp(\F,\reparamspace) : \F \in \lin_T^c(\cH)
	\racc
\]
with
\[
	\lin_T^c(\cH)
	= 
	\lacc
		\h_0 +  \sum_{t=1}^{T} \alpha_t \h_t: \; \lp \alpha_1, \ldots, \alpha_T \rp \in \R_+^{T}, \sum_{t=1}^T \alpha_t \leq Tc, \; \h_0, \ldots, \h_T \in \cH
	\racc.
\]
For $\F, \G : \suppX \mapsto \reparamspace$, we will use the notation
\[
    \left\{
    \begin{array}{lll}
        \nabla_{\bphi} \hR_n^{\hu}(\F) &= \frac{1}{k_n} \sum_{i=1}^n \nabla_{\bphi} L \lp \F(\X_i), Y_i - \hu(\X_i ; \tfrac{k_n}{n}) \rp \I \lacc  Y_i > \hu(\X_i ; \tfrac{k_n}{n}) \racc \\
        \nabla_{\bphi} R^u(\F)  &=  \frac{n}{k_n} \E \lc  \nabla_{\bphi} L \lp \F(\X), Y - u(\X ; \tfrac{k_n}{n}) \rp \I \lacc  Y > u(\X; \tfrac{k_n}{n}) \racc \rc \\
        \nabla_{\bphi} R(\F) &= \E \lc \nabla_{\bphi} L(\F(\X), Z) \rc
    \end{array}
    \right.
\]
for the derivative of the different risks.

\begin{theorem}
\label{thm:stochastic}
	Let $\hF_n \in \outputspace$ be such that $\nabla_{\bphi} \hR_n^{\hu}(\hF_n) = (0,0)^\tr$ and let $\F^u \in \outputspace$ be such that $\nabla_{\bphi} R^u(\F^u) = (0,0)^\tr$. Under Assumptions \ref{ass:tail-of-Y}  to \ref{ass:sv-function}, there exists a constant $\cst > 0$, not depending on $n$, such that for any $\delta \in (0,1)$ and for any $n \in \N_0$ which is large enough such that $a_n(\delta) \leq 1/2$ and $\sup_{\x \in \suppX} \labs r(u(\x ; \tfrac{k_n}{n}) / 2 \mid \x) \rabs \leq 1$ and $\sup_{\x \in \suppX} \bF_{Y \mid \X = \x}(u(\x ; \tfrac{k_n}{n}) / 2) \leq 2k_n/n$, there exists an event with probability at least $1-\delta$ on which
	\[
		\int_{\suppX} \| \hF_n(\x) - \F^u(\x) \|_\infty \diff P_{\X}(\x)
		\leq
		\cst \lacc B_{n,1}(\delta/6) + B_{n,1}(\delta/6) \racc,
	\]
	where
	\[
		B_{n,1}(\delta)
		= 
        \log k_n
        \lp
        \sqrt{ \frac{a_n(\delta)}{k_n} \log(1/\delta)}
            +
        a_n(\delta)
        \rp,
	\]
	and
	\[
		B_{n,2}(\delta)
		= \frac{\log k_n}{\sqrt{k_n}} \lc (1+Tc)
        + \sqrt {\log(1/\delta)} \rc.
	\]
\end{theorem}

The term $B_{n,1}$ is associated with the estimation of the intermediate quantile while the term $B_{n,2}$ is associated with the empirical estimation of the pre-asymptotic risk $R^u$. 
For fixed $T>0$, the final order of the bound in terms of the sample size depends on the speed at which the sequence $(a_n(\delta))_{n \in \N}$ converges to zero with $n \to \infty$. If it converges at the rate $1/\sqrt{k_n}$, both terms are of the same order $\log k_n / \sqrt{k_n}$ which is the same as the one obtained in \citet{farkas2024gptree}. In practice, $T = T(n)$ and the term $B_{n,2}$ is converging to zero at a slower rate than the one obtained in \citet{farkas2024gptree}.

It seems from Theorem \ref{thm:stochastic} that increasing the number of boosting iterates $T$ always lead to an increase in the error which may be surprising, but of course doing so reduces the approximation error discussed in Section \ref{subsec:approx}. In practice we choose $T = T(n)$ in a way to get enough predicting power (low bias) while trying to not overfit the training data (small variance).

The dimension $d$ of the predictor space does not appear explicitly in the bound but appear indirectly through the VC dimension of the classes of hypothesis (Assumption \ref{ass:hypothesis}) which typically increases with $d$. We discuss this in Section \ref{sec:numerical} for the case of decision trees.

\subsection{Bias error}
\label{subsec:bias}

We now turn our attention to the fact that the risk $R^u$ is only a pre-asymptotic approximation of the true risk $R$. As such, minimizing this quantity leads to a bias which may be controlled under second-order regular variation which we state below. The condition corresponds to Assumption 3 in \citet{farkas2024gptree}, to condition C6 in \citet{beirlant2004local} and is related to the theory of slowly varying functions with remainder condition \citep{goldie1987slow}.

\begin{assumption}
\label{ass:second-order-RV}
	For any $\x \in \suppX$ and $y>1$, there exists $c(\x) \in \R$, $\rho(\x) < 0$ and a strictly positive function $\psi(\cdot\mid \x)$ with $\psi(t\mid \x) \to 0$ as $t\to \infty$ such that
	\[
	\frac{\ell(ty\mid \x)}{\ell(t\mid \x)} 
	= 
	1 + c(\x)\psi(t\mid \x) \int_1^y u^{\rho(\x)-1} \diff u +
	\operatorname{o}(\psi(t\mid \x)), \qquad t \to \infty,
	\]
	where $\ell$ is the slowly varying function appearing in Assumption \ref{ass:tail-of-Y}. We also assume that
	\[
		C = \sup_{\x \in \suppX} |c(\x)| < \infty 
	\]
	and that 
	\[
		\Psi(t) = \sup_{\x \in \suppX} \psi(t\mid \x) \to 0,
		\qquad t \to \infty.
	\]
\end{assumption}

\begin{theorem}
\label{thm:bias}
	Let $\F^u \in \outputspace$ be such that $\nabla_{\bphi} R^u(\F^u) = (0,0)^\tr$ and let $\F_* \in \outputspace$ be such that $\nabla_{\bphi} R(\F^u) = (0,0)^\tr$. Under Assumptions \ref{ass:fisher} and \ref{ass:second-order-RV}, we have
	\[
		\int_{\suppX} \| \F^u(\x) - \F_*(\x) \|_\infty \, \diff P_{\X}(\x) \leq \ist B(n),
	\]
	where $\ist > 0$ is defined in Assumption \ref{ass:fisher} and
	\[
		B(n) = \operatorname{O} \lp \sup_{\x \in \suppX} \Psi(u(\x;\tfrac{k_n}{n})) \rp, \qquad n \to \infty.
	\]
\end{theorem}

Note that in comparison to Theorem \ref{thm:stochastic}, the asymptotic nature of Theorem \ref{thm:bias} comes from Assumption \ref{ass:second-order-RV} which is stated in asymptotic manner. The control on the integral of interest in completely explicit in the proof.

\subsection{Approximation error}
\label{subsec:approx}

As the number of boosting iterates $T = T(n) \to \infty$, the bias between an optimal function inside the output of the boosting algorithm (Algorithm \ref{alg:GB}) $\F_*$ as in Theorem \ref{thm:bias} and the true function $\F_0 = (\nu_0, \xi_0)$ typically decreases as we are able to approximate more general function via $\outputspace$ whose size is increasing. The counterpart is that the variance increases as suggested in Theorem \ref{thm:stochastic}. This suggests a standard bias-variance tradeoff depending on the size of the class $\outputspace$. In practice, this tradeoff isn't solvable analytically but $T = T(n)$ will be considered as an hyperparameter and will be chosen by cross-validation. Indeed, for trees which are often used as the default choice for the class $\cH$, no rate of approximation is available for $\int_{\suppX} \| \F_*(\x) - \F_0(x) \|_\infty \, \diff P_{\X}(\x)$, see Lemma 5 in the supplement to \citet{biau2021optimization}.

\section{Numerical results}
\label{sec:numerical}

The choice of trees as base classifiers corresponds to the Gradient Boosting for EXtremes (GBEX) method introduced in \cite{velthoen2023gradient}, re-parametrized using the transformation
\[
    \btheta = (\sigma,\xi) \in \paramspace \mapsto \bphi = (\nu,\xi) \in \reparamspace
\]
making the Fisher information matrix of the GP distribution $I(\nu,\xi)$ diagonal \eqref{eq:fisher-true}.
Based on their code available from the R package \url{https://github.com/JVelthoen/gbex/}, we implement the gradient tree boosting procedure for the GP distribution parameters in this new parameter space. This is an easy adaptation from the original R package GBEX and the adapted code may be asked directly to the corresponding author.

Note that Assumption \ref{ass:hypothesis} is satisfied when binary decision trees are used as proposed here. Indeed, it is shown in Corollary 10 of \cite{leboeuf2020trees} that the VC dimension of the class of binary decision trees with a structure containing $N$ internal nodes based on $d$-valued inputs is finite and behaves as $\operatorname{O}(N\log(dN))$ when $N \to \infty$. Often, decision stumps are used where there is no internal node and each tree has just two leaves and in this case it is shown in Corollary 8 of \cite{leboeuf2020trees} that the VC dimension is the largest $m \in \N_0$ such that
\[
    \binom{m}{ \lfloor m/2 \rfloor } \leq 2d.
\]
Furthermore Assumption \ref{ass:fisher} reduces to Assumption 2 of \cite{farkas2024gptree} when considering decision trees as hypothesis classes for the parameters.

Section \ref{subsec:simus} will be devoted to compare numerically the two parametrization for the gradient boosting with GP loss and Section \ref{subsec:application} will be concerned with the application of the method to medical malpractice data.

\subsection{Simulations: comparison with standard GBEX}
\label{subsec:simus}

We start by investigating the impact of the reparametrization on level curves of the GP model to illustrate how it can help improving stability during the learning process as already observed by \citet{pasche2024neural}. To this end, we simulate $n = 2000$ observations from the GP distribution with scale $\sigma = 1$ and shape $\xi = 0.5$ and draw the contour lines of the associated negative log-likelihood. Results are displayed on Figure \ref{fig:curves}. We clearly see how the elliptical contour lines in the standard parametrization become more isotropic in the new parametrization suggesting a significant reduction in parameter correlation. 
\begin{figure}[h]
    \centering
    \includegraphics[width=0.48\linewidth]{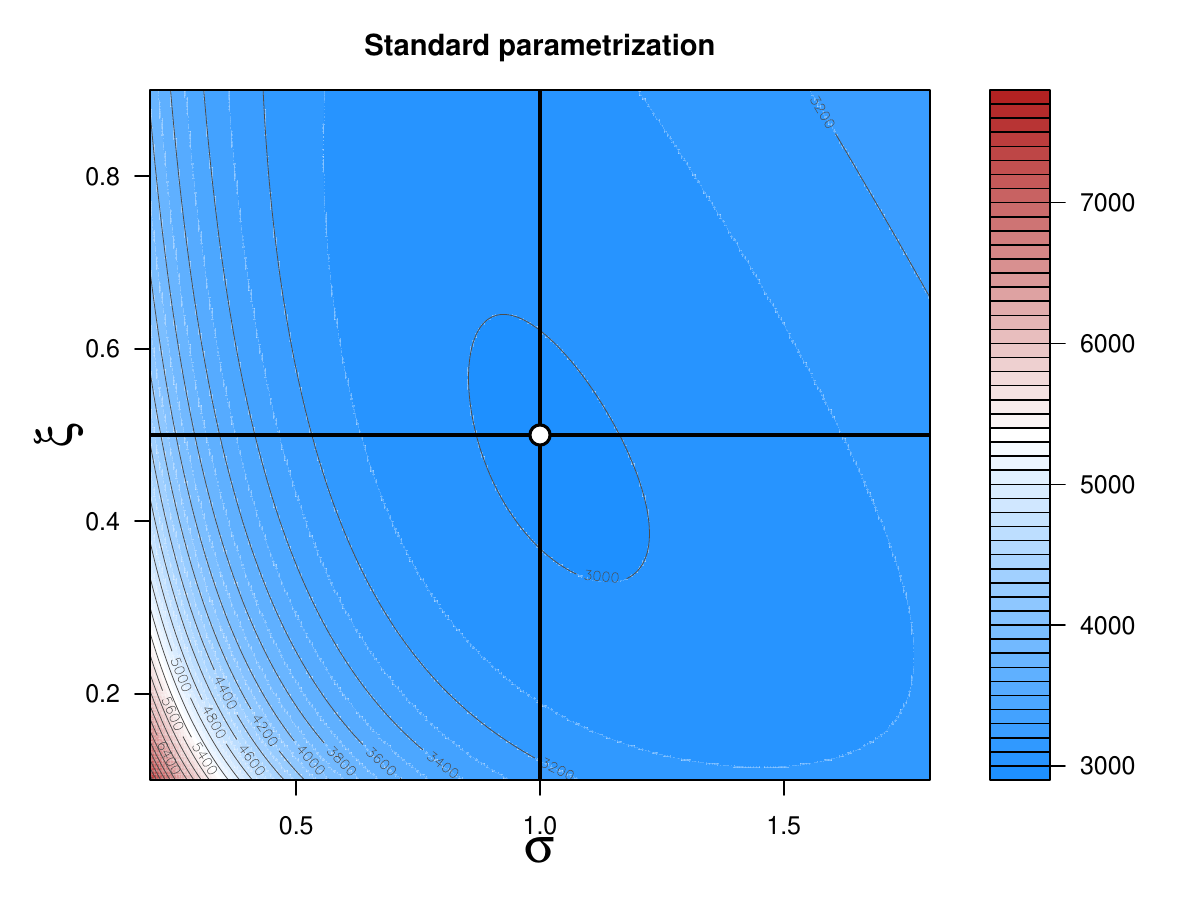}
    \includegraphics[width=0.48\linewidth]{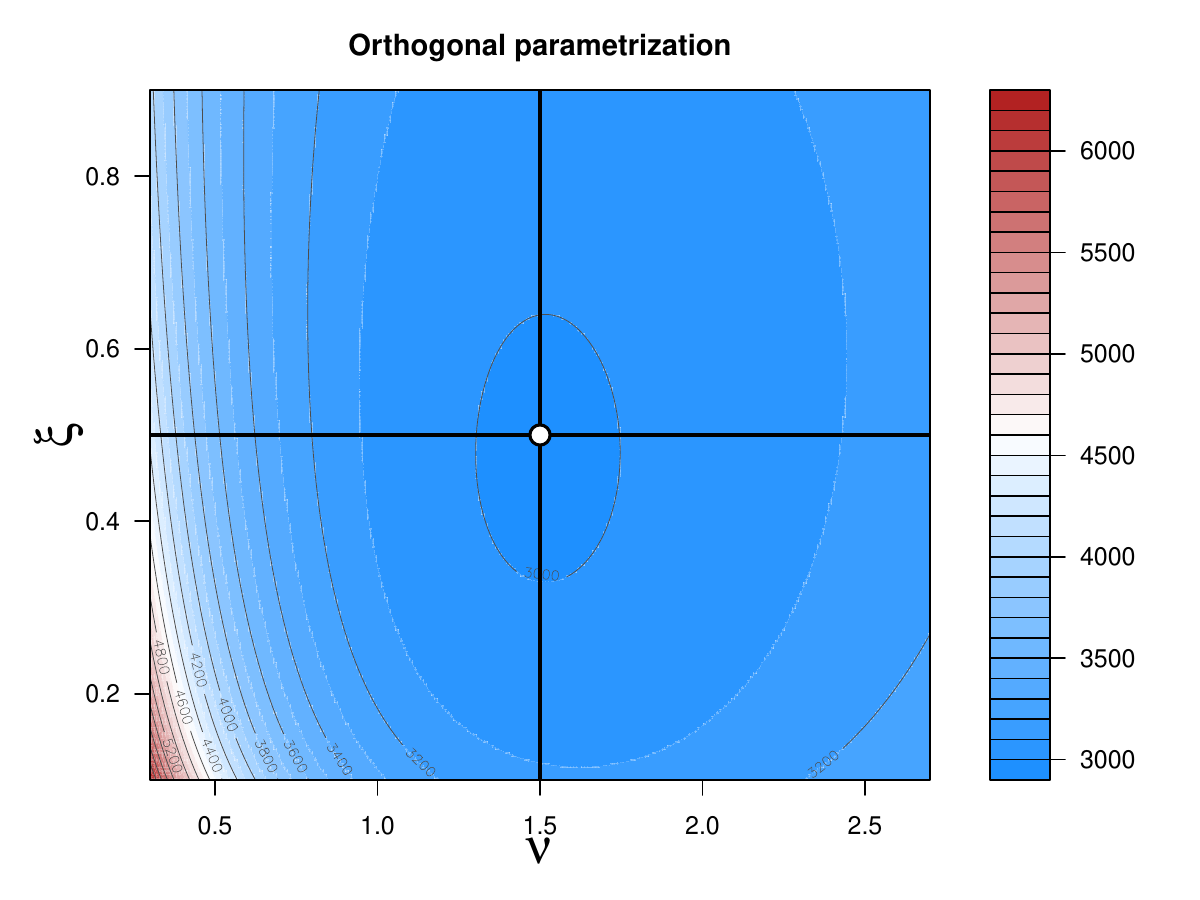}
    \caption{Impact of the reparametrization on level curves from the GP model.}
    \label{fig:curves}
\end{figure}

To investigate this phenomena further, we run the following experiment $200$ times. We simulate $n=2000$ data from the simple model $Y \sim \operatorname{GP}(1, \xi(X))$ with $X \sim \operatorname{Uniform}[0,1]$ and
\[
    \xi(x)
    =
    \begin{cases}
        0.3 &\text{if } x \in [0,0.3] \\
        0.2 &\text{if } x \in (0.3,0.7) \\
        0.1 &\text{if } x \in [0.7,1]
    \end{cases}.
\]
We then estimate the parameters with GBEX in the two different parametrization and keep track of the evolution of the correlation between the gradients as the number of boosting iterates increases. We use the same hyperparameters for each model: $500$ boosting iterates, trees with depth one for each parameter and fixed learning rates of $0.01$ for the $\sigma$ (or $\nu$) and $0.001$ for $\xi$. On Figure \ref{fig:corr-curves}, we display the evolution of the average (on the 200 experiments) correlation between the parameter gradients during the learning process. These figures clearly indicate that the reparametrization helps in de-correlating the gradients and thus improve training stability.
\begin{figure}[h]
    \centering
    \includegraphics[width=0.48\linewidth]{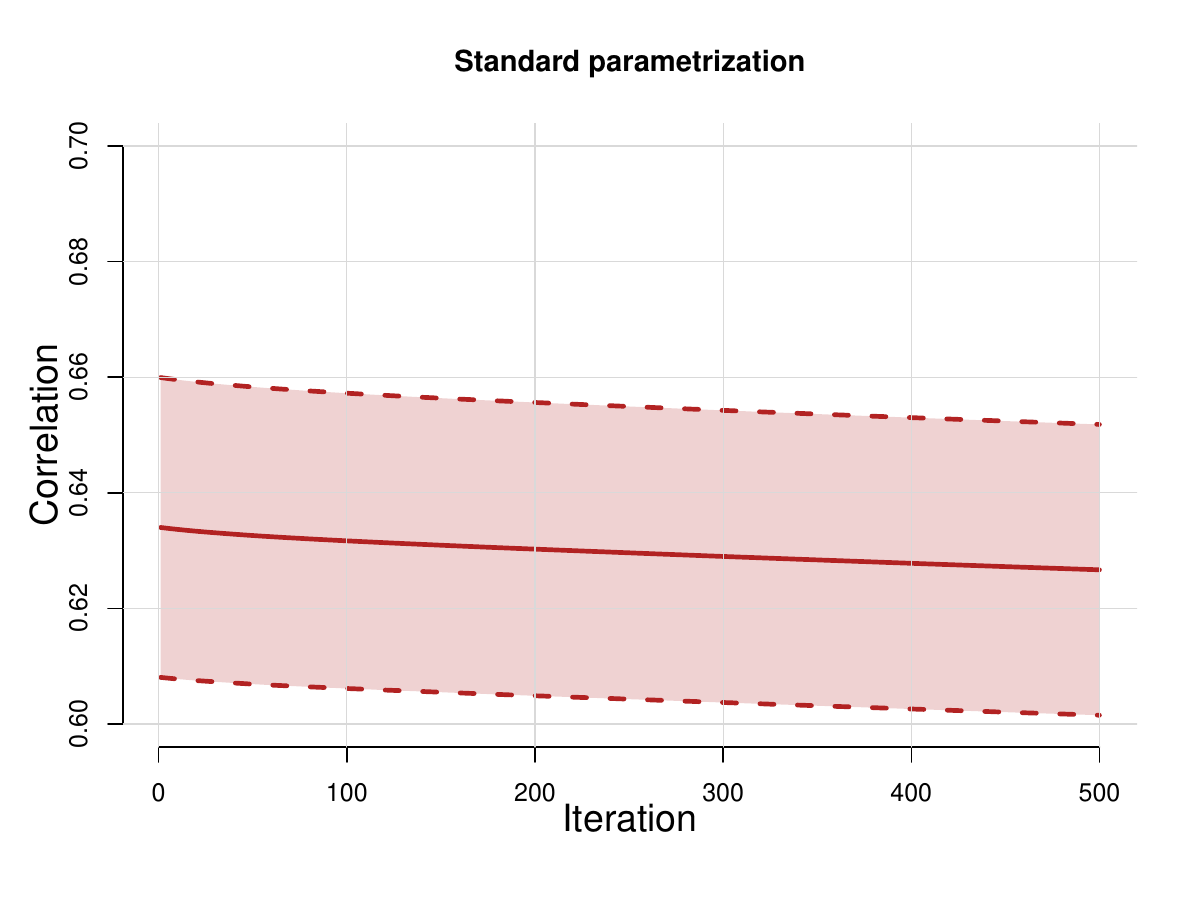}
    \includegraphics[width=0.48\linewidth]{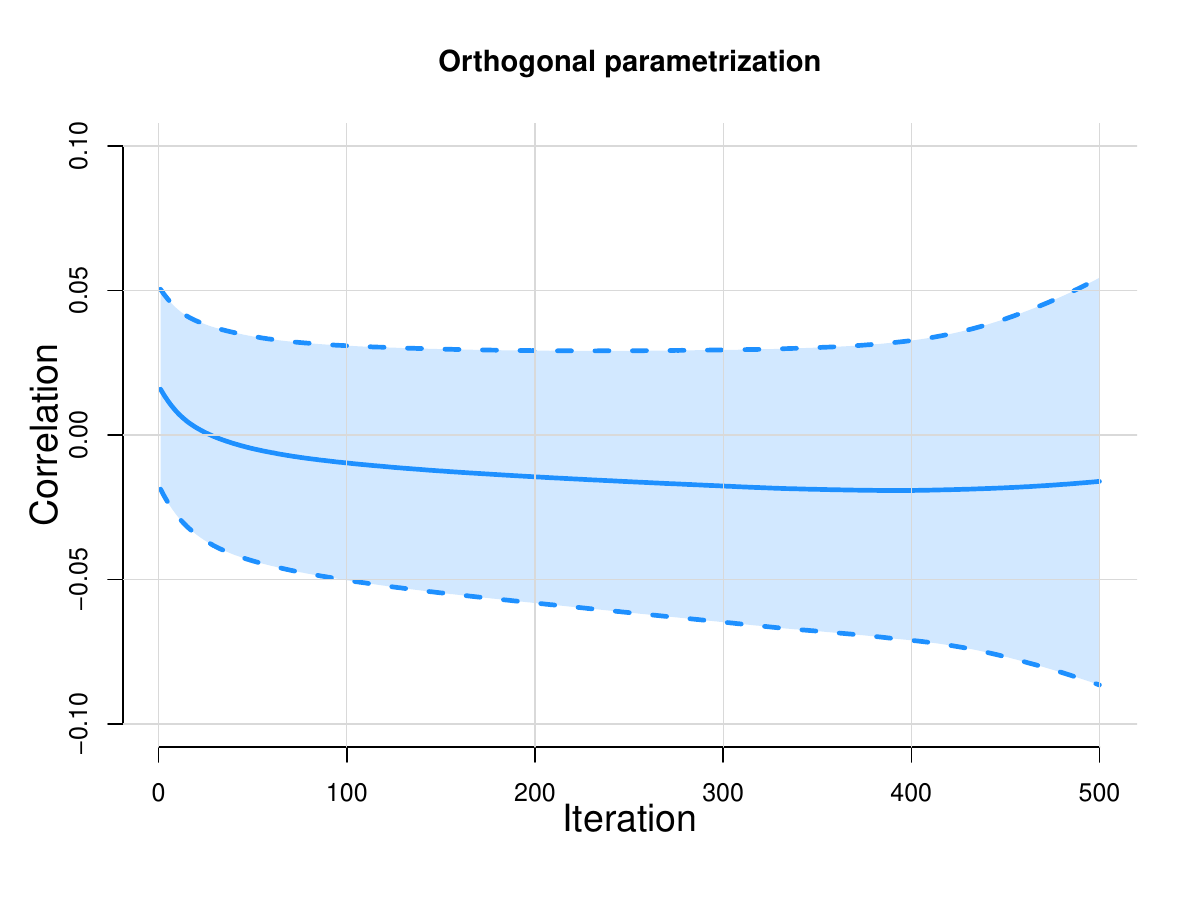}
    \caption{Mean correlation of the parameter gradients during training. Areas in light are obtained by adding and removing one standard deviation (obtained on the 200 experiments).}
    \label{fig:corr-curves}
\end{figure}

\subsection{Application: medical malpractice insurance}
\label{subsec:application}

\paragraph{Data presentation.}

To illustrate the estimation procedure, we consider a real dataset gathering closed claim data in medical malpractice insurance, available from the Texas Department of Insurance (TDI) on \url{https://dataverse.harvard.edu/dataverse/georgetownlaw}. The TDI requires all malpractice carriers to submit reports of commercial liability closed claims involving bodily injury settled under Texas law. Malpractice carriers are required by TDI to submit a detailed form for claims with settlements of $25 000 \$$ (nominal) or more and we focus on these claims. We consider the year 1990--2008 which means that we have access to $n = 18886$ closed medical professional liability claims. The dollar values have to be corrected for inflation on this large period of time and we follow the procedure of Section 4.2 in \citet{lopez2019censored} since the inflation of the amounts in medical malpractice data does not follow common consuming products.

Our goal is to analyze the tail of the final settlement costs based on the GP model, and its relation to known (or estimated) quantities early during the claim process so that the insurer is able, from those information, to properly assess the risk of the given claim. We rely on the following covariates.
\begin{itemize}
    \item The number of days before reporting the incident to the primary insurer.
    \item The number of days before settlement. 
    \item Information about the victim (age, employment status).
    \item Information about the kind of injury.
    \item Information about the policy (line of business insurance, policy limit).
    \item The initial reserve. It corresponds to a ``first guess'' by an expert at the beginning of the claim process which tries to assess the final cost of the claim. We will check whether this amount is indeed positively related to the heaviness of the claim.
\end{itemize}
Summary statistics of the settlement costs and the covariates are displayed in Table \ref{tab:summary_stats}. The monetary values for the quantitative variables are expressed in (inflation corrected) US dollars. For the qualitative variables, when too many levels are present, we only show the ones containing the largest part of the data.

\begin{table}[htbp]
  \centering
  \caption{Summary statistics of the TDI medical malpractice dataset 1990--2008}
  \label{tab:summary_stats}
  \footnotesize
  \begin{tabular}{@{}l l c c c c@{}}
    \toprule
    \textbf{Variable} & \textbf{Level} & \textbf{Min} & \textbf{Max} & \textbf{Median} & \textbf{Mean / N (\%)} \\
    \midrule
    \textit{Quantitative variables} & & & & & \textit{Mean} \\
    Settlement cost & -- & 19,472 & 20,379,160 & 169,826 & 424,596 \\
    Initial reserve & -- & 0 & 6,269,657 & 66,278 & 122,040 \\
    Policy limit & -- & 0 & 53,338,328 & 653,258 & 1,076,331 \\
    Victim age (in years)  & -- & 0.08 & 110 & 42.00 & 41.04 \\
    Days to settlement & -- & 7 & 10,332 & 1,204 & 1,359 \\
    Days to report & -- & 0 & 9,391 & 473 & 571.4 \\
    \midrule
    \textit{Qualitative variables} & & & & & \textit{N (\%)} \\
    Employment status & Employed & -- & -- & -- & 1,450 (41.3\%) \\
                      & Unemployed & -- & -- & -- & 2,061 (58.7\%) \\
    \addlinespace
    Injury type  & Death & -- & -- & -- & 6,428 (34.0\%) \\
                      & Brain damage     & -- & -- & -- & 1743 (9.0\%) \\
                      & Eye injury     & -- & -- & -- & 496 (3.0\%) \\
                      & Others              & -- & -- & -- & 10,198 (54\%) \\
    \addlinespace
    Activity when injured   & Surgical/medical care         & -- & -- & -- & 17,773 (94.0\%) \\
                      & Falls           & -- & -- & -- & 496 (3.0\%) \\
                      & Others             & -- & -- & -- & 617 (3.0\%) \\
    \addlinespace
    Business class   & Physicians and surgeons         & -- & -- & -- & 13,593 (72.0\%) \\
                      & Hospital           & -- & -- & -- & 2,927 (15.0\%) \\
                      & Others             & -- & -- & -- & 2,366 (13.0\%) \\
    \bottomrule
  \end{tabular}
\end{table}

To apply the gradient boosting approach, we first need to obtain the excesses above a threshold from which the GP model seems appropriate. Even though in the theory---and in \citet{velthoen2023gradient}---excesses are of the form $Z_i = Y_i - \hu(\X_i ; \tfrac{k_n}{n})$ for an estimated intermediate quantile function $\hu$, this method does not seem the most appropriate to us in practice since no visual plot is available to check that some form of stability is attained from that threshold in the estimation of the GP distribution. We therefore chose to rely on a different approach. Based on the GPtree method in \citet{farkas2021cyber,farkas2024gptree}, we work with a locally constant threshold for which we are able to, at least visually, inspect that the tail model seems reasonable.

\paragraph{Threshold selection.}

We start by choosing a constant threshold $u$ such that the fit of a GP distribution on the excesses
\[
    \lacc
        Z_i = Y_i - u : i \in [n] \text{ and } Y_i > u
    \racc
\]
is approximately stable. To do so, we rely a plot of $u \mapsto \xi(u)$ for different values of the threshold and a mean residual life plot. If the chosen threshold is acceptable, the first plot should exhibit some stability above the chosen value of $u$ and the second plot should exhibit linearity. The results are displayed in Figure~\ref{fig:tc-plots}.
\begin{figure}[h]
    \centering
    \includegraphics[width=0.48\linewidth]{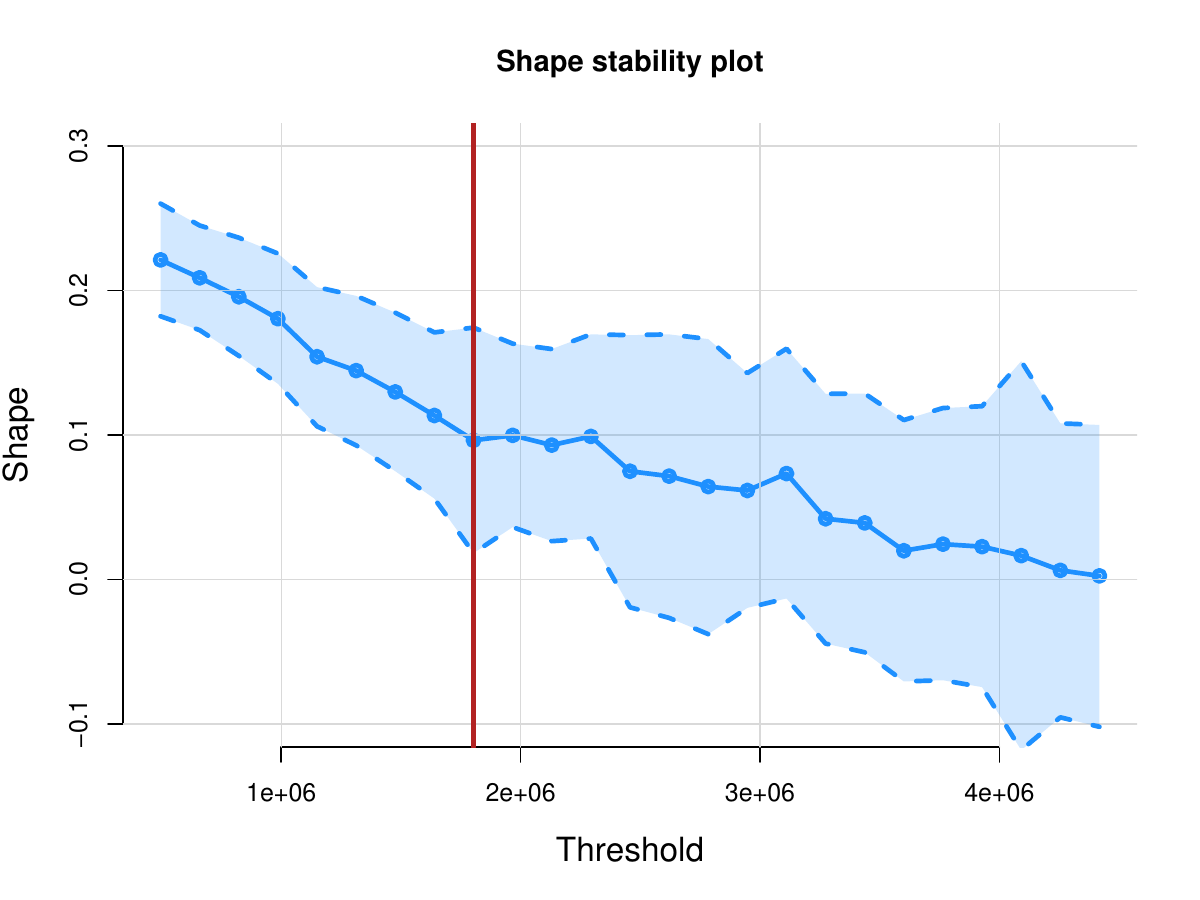}
    \includegraphics[width=0.48\linewidth]{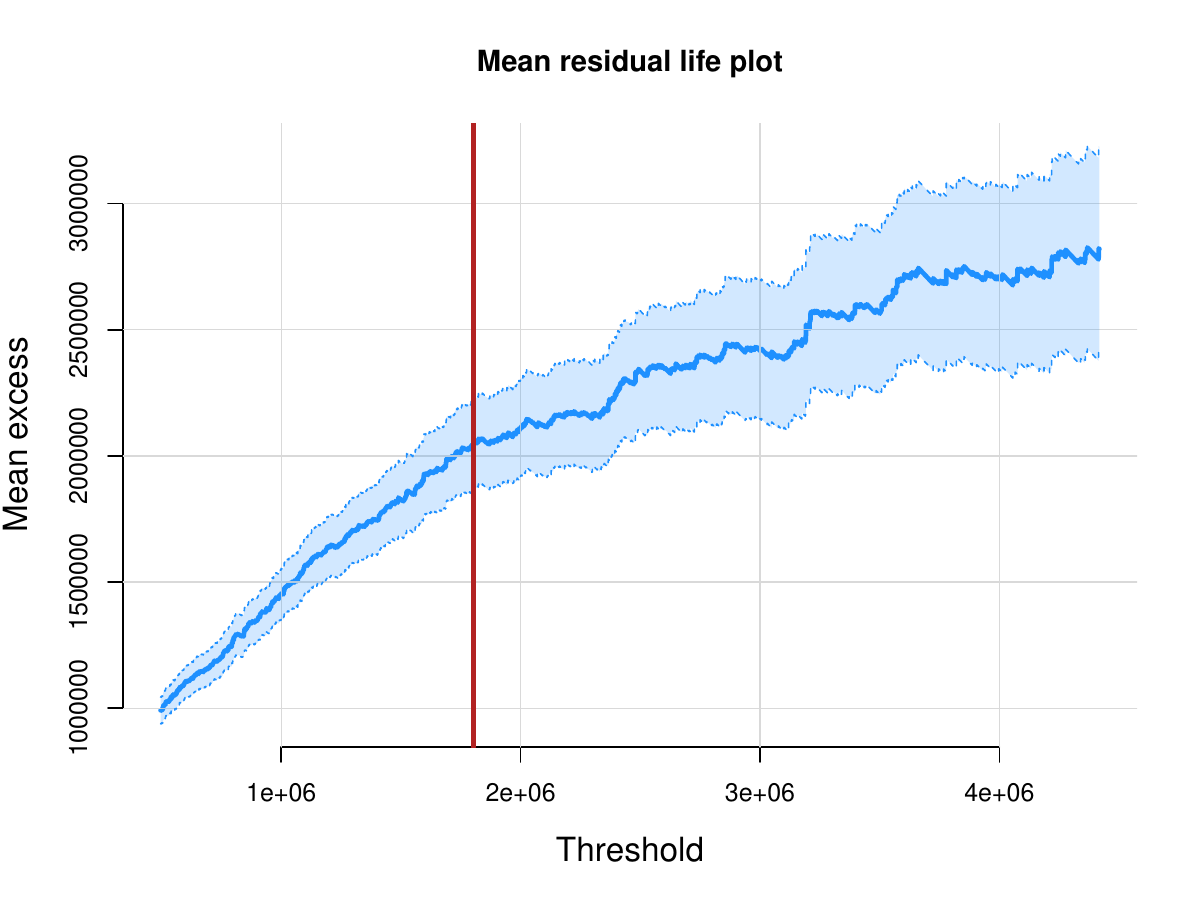}
    \caption{Threshold selection plots. Left: stability plot for the shape parameter $\xi$. Right: mean residual life plot. Vertical lines correspond to the chosen threshold $u=1 803 255.19 \$$.}
    \label{fig:tc-plots}
\end{figure}
We have $758$ excesses above the threshold $u$ which we will use in order to fit the locally constant threshold. 

The tree obtained by the GPtree method is displayed on Figure \ref{fig:gptree} and provides us with 6 regions of the space on which we try to find a threshold above which the GP model is reasonable. A stability procedure similar to what has been done previously leads to the thresholds displayed on Figure \ref{fig:thresholds}. The quantile-quantile plots for the fit of the GP distribution in each leaf of the tree are shown on Figure \ref{fig:qq-leafs} in Appendix \ref{app:figures}. Of course, theses fits are far from being perfect but they will be largely improved by the gradient boosting procedure that we will apply below.

\begin{figure}[h]
    \centering
    \includegraphics[width=0.75\linewidth]{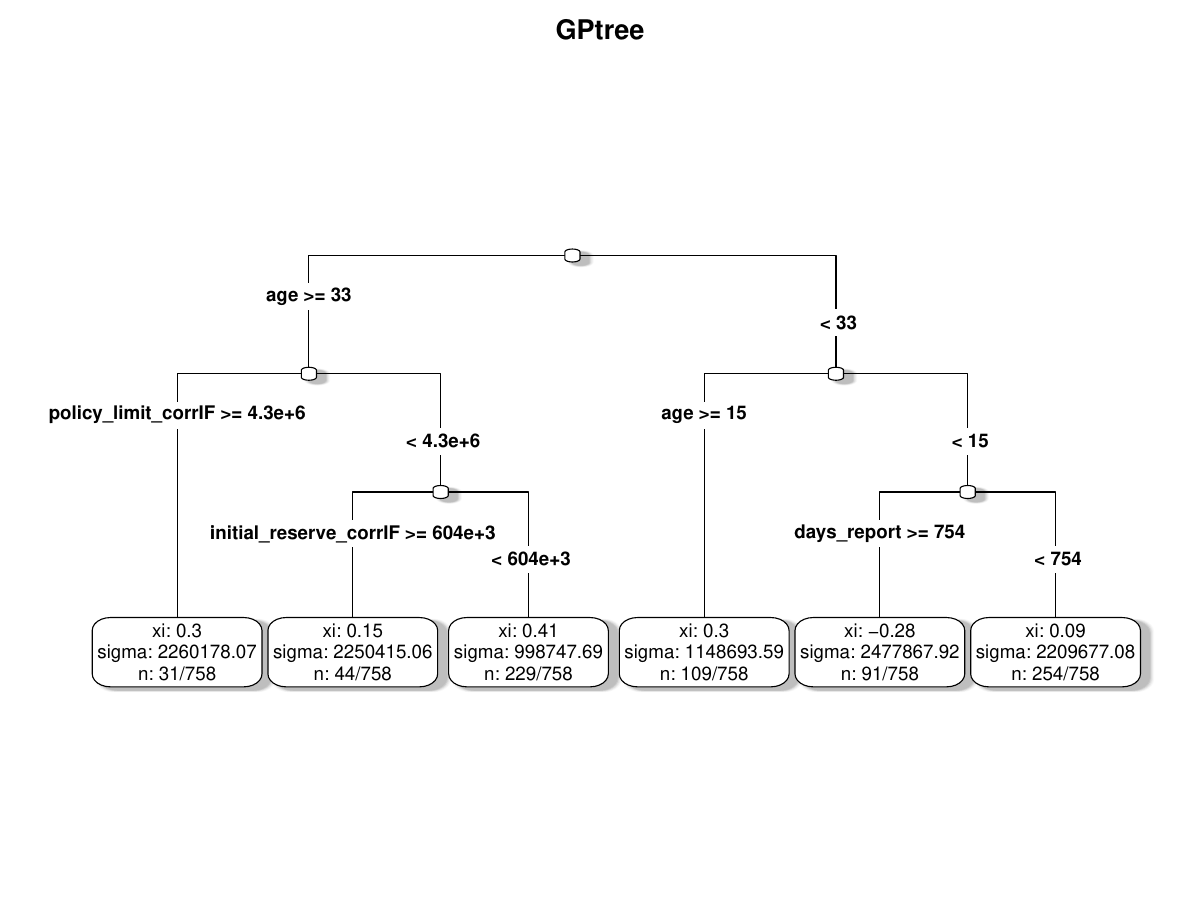}
    \caption{The leaves produced by the GPtree method which will form the basis in order to pick a locally constant threshold.}
    \label{fig:gptree}
\end{figure}
\begin{figure}[h!]
    \centering
    \includegraphics[width=0.75\linewidth]{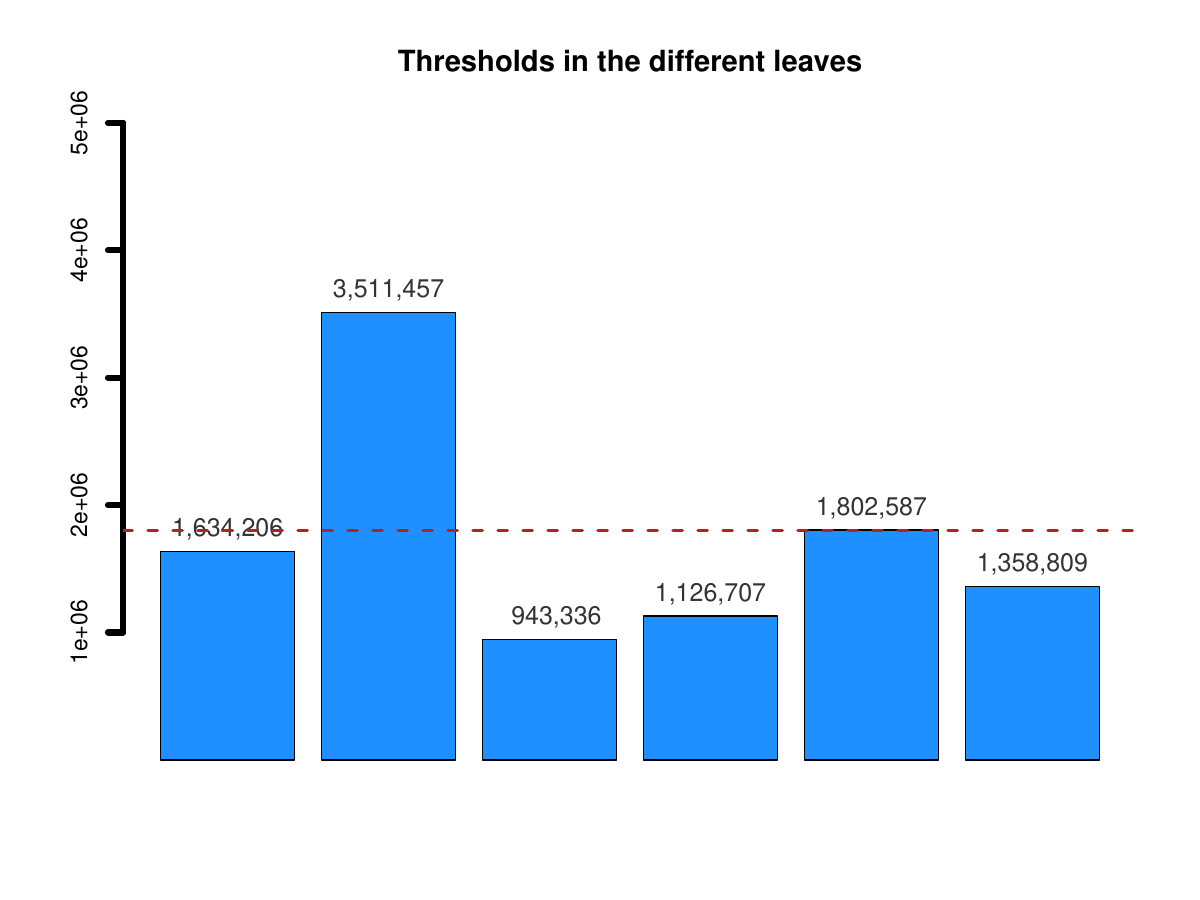}
    \caption{The different thresholds in each leaf. The order of the leaves is the same as on Figure \ref{fig:gptree}. Horizontal line corresponds to the initial constant threshold $u$.}
    \label{fig:thresholds}
\end{figure}

Note that refining the constant threshold to a locally constant one helps in allowing the use of more data points to fit the gradient boosting procedure as, on some leaves, a smaller threshold may be chosen, as seen on Figure \ref{fig:thresholds}. This is helpful as the gradient boosting algorithm may fail to converge if too little training data is provided. We now have $1417$ excesses above our locally constant threshold.

\paragraph{Gradient boosting results.}

The hyperparameters of the gradient boosting procedure are all chosen by cross-validation with $5$ folds and the procedure is repeated $5$ times. The obtained values are described in Table \ref{tab:hyperparams}.

\begin{table}[h]
\centering
\begin{tabular}{lcc}
\toprule
\textbf{Hyperparameter} & \textbf{$\nu$} & \textbf{$\xi$} \\
\midrule
Number of boosting iterations ($T$) & \multicolumn{2}{c}{71} \\
Tree depth & 3 & 3 \\
Learning rate & 0.02 & 0.002 \\
Minimum leaf size & 10 & 10 \\
\bottomrule
\end{tabular}
\caption{Boosting model hyperparameters.}
\label{tab:hyperparams}
\end{table}

To check if the obtained GP model provides a good fit for the tail of the settlement costs in the database, we rely on the GP quantile-quantile plot displayed on Figure \ref{fig:qq-final}.
\begin{figure}[h]
    \centering
    \includegraphics[width=0.75\linewidth]{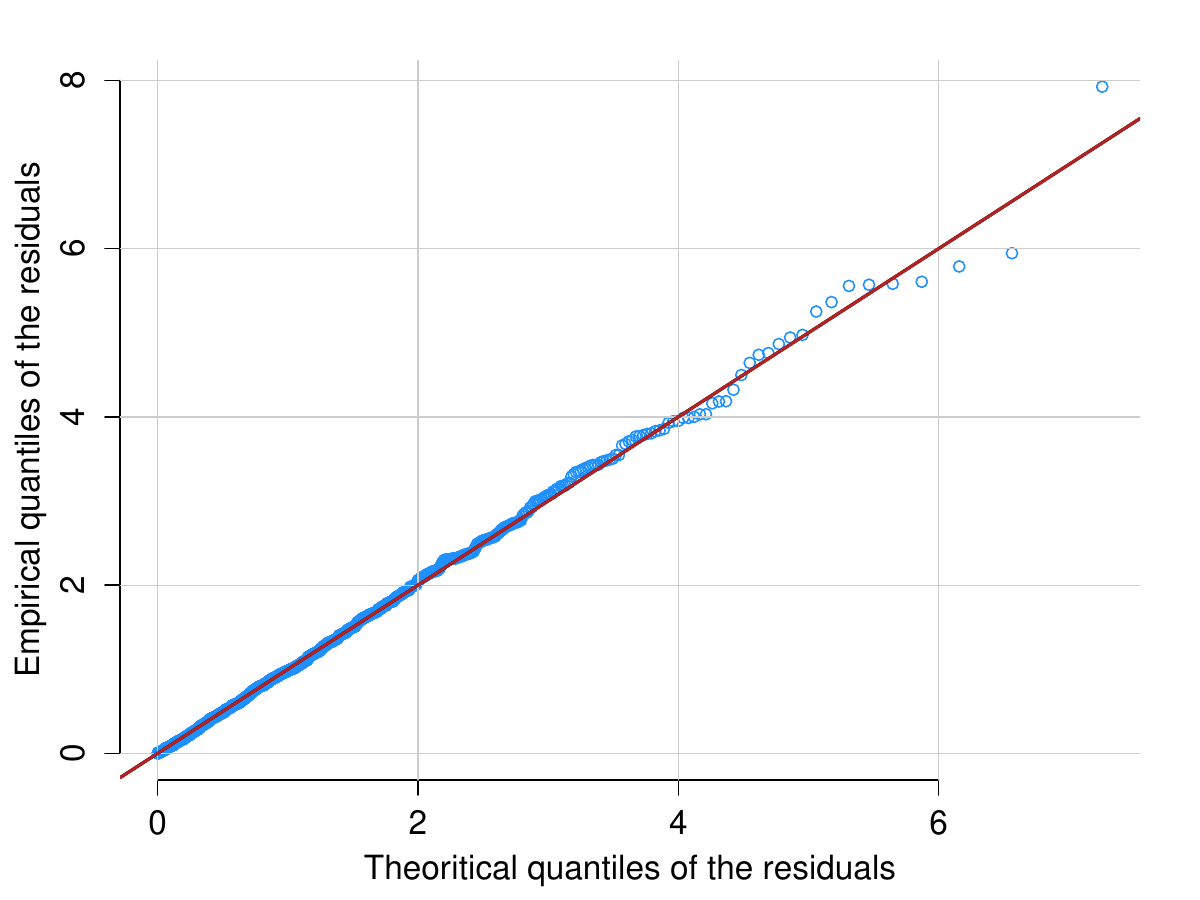}
    \caption{GP quantile-quantile plot associated with our final model.}
    \label{fig:qq-final}
\end{figure}
Given that the points fit the diagonal quite closely, we conclude that the GP model seems very appropriate in order to fit the tail of the settlement costs in the context of bodily injured claims in medical liability insurance. 

One feature of of primary importance for the insurance sector concerns the size of the possible extreme costs. If the estimated parameters $\xi(\x) > 1$, no mutualization is possible since the expected value of the associated claim distribution is infinite and the law of large number does not hold. In our case, as illustrated in Figure \ref{fig:shapes}, the estimated shape parameters in the database are all smaller than $0.3$ meaning that the extremes are relatively ``small'' and that standard insurance is indeed maintainable. Note that for some values of $\x$, we even obtain negative values, suggesting a finite upper-end-point for the associated settlement costs $Y \mid \X = \x$.
\begin{figure}[h!]
    \centering
    \includegraphics[width=0.75\linewidth]{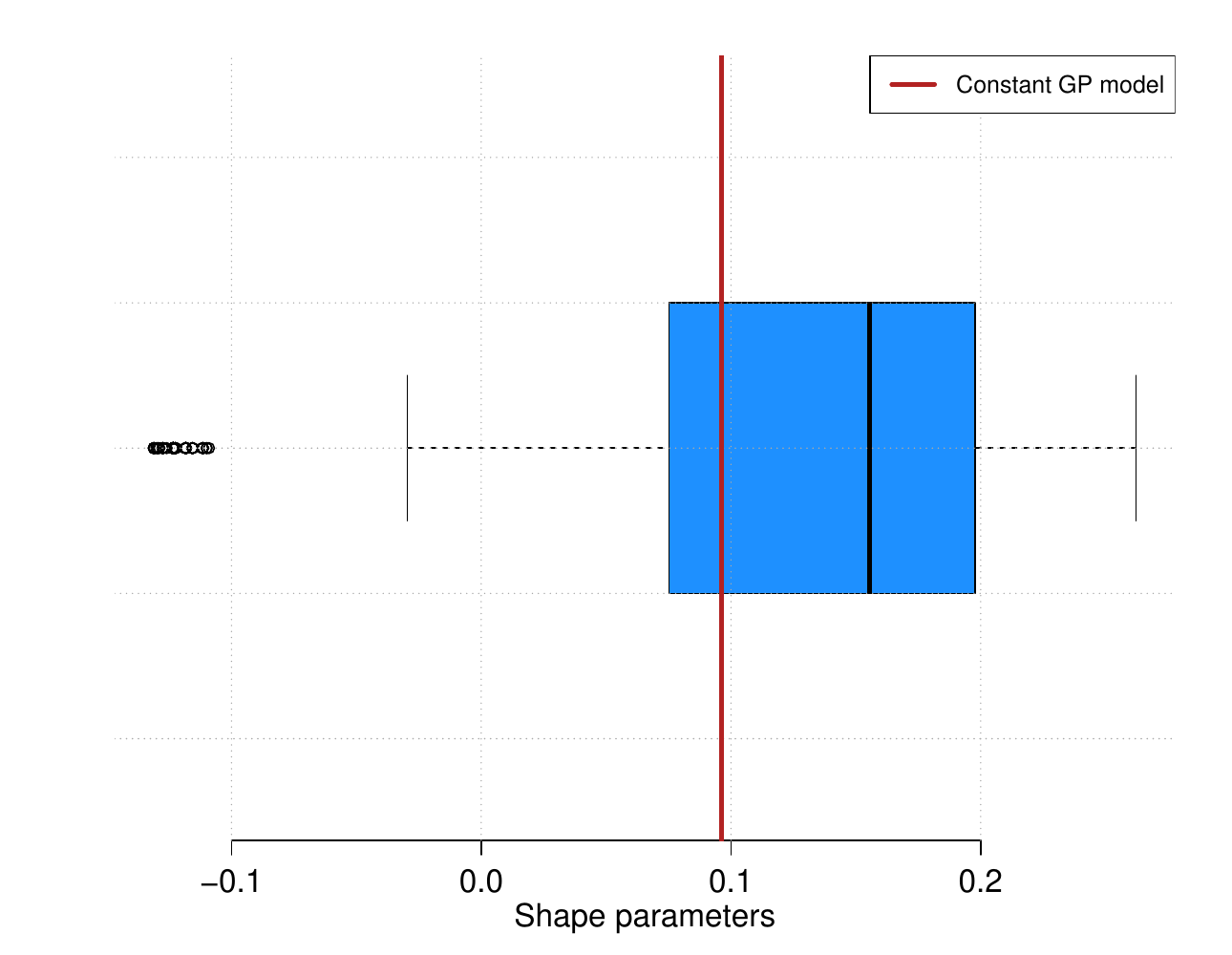}
    \caption{Boxplot of the estimated shape parameters $\xi(\x)$ in the data.}
    \label{fig:shapes}
\end{figure}

To understand which variables played the biggest role in order to assess the heaviness of the tail of the settlement costs, helpful tool  variable importance plot, as advised in \citet{velthoen2023gradient}. The relative importance plot is displayed in Figure \ref{fig:VIrel} The bars are scaled to 100 and the highest ones corresponds to the most influent covariates in the learning process. 
\begin{figure}[H]
    \centering
    \includegraphics[width=0.7\linewidth]{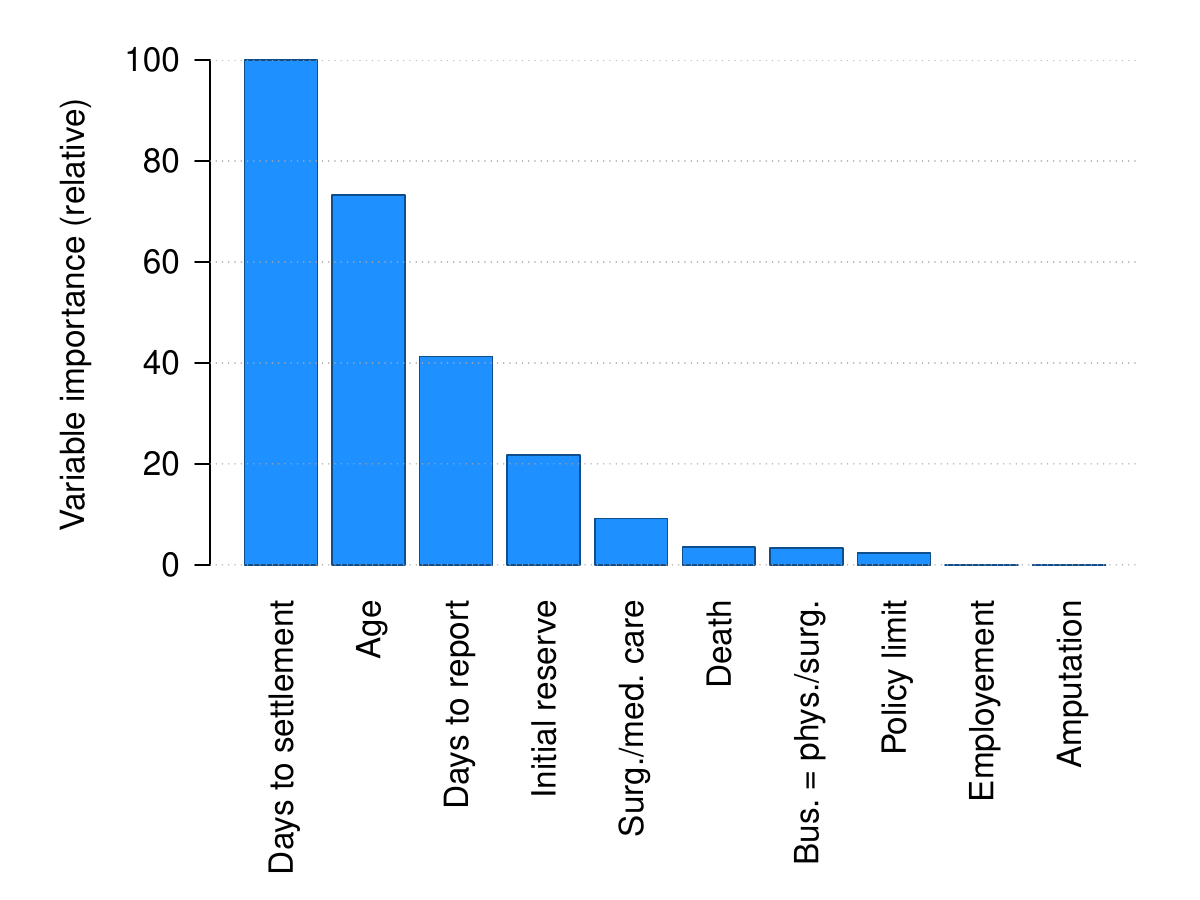}
    \caption{Relative variable importance plot (10 most important covariates).}
    \label{fig:VIrel}
\end{figure}
As expected from \citet{lopez2019censored}, the time to settlement plays the biggest role in determining the size of the tail of the settlement cost distribution. The age of the victim as well as the number of days between the incident and its reporting to the insurer also appear to be determinant factors. The initial reserve predicted by the expert is also an important factor. Whether the victim died from the injury also has an impact, as one may expect the settlement cost to be lower than in cases of long-term injuries.

To investigate further the dependence between these covariates and the settlement cost of the claims, we consider partial dependence plots. The results are displayed on Figure \ref{fig:pd-plots}.

\begin{figure}[ht!]
    \centering
    \includegraphics[width=0.49\linewidth]{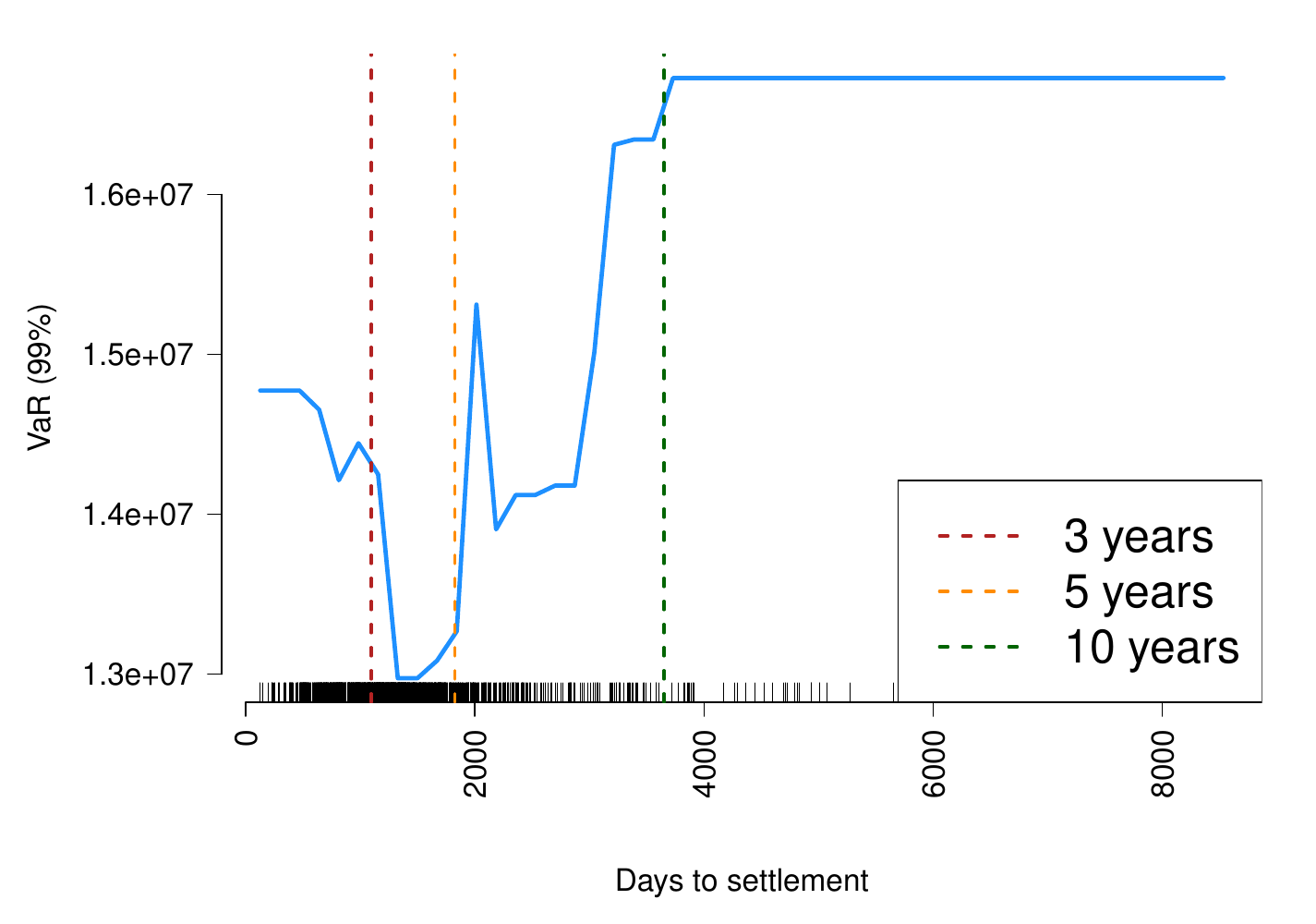}
    \includegraphics[width=0.49\linewidth]{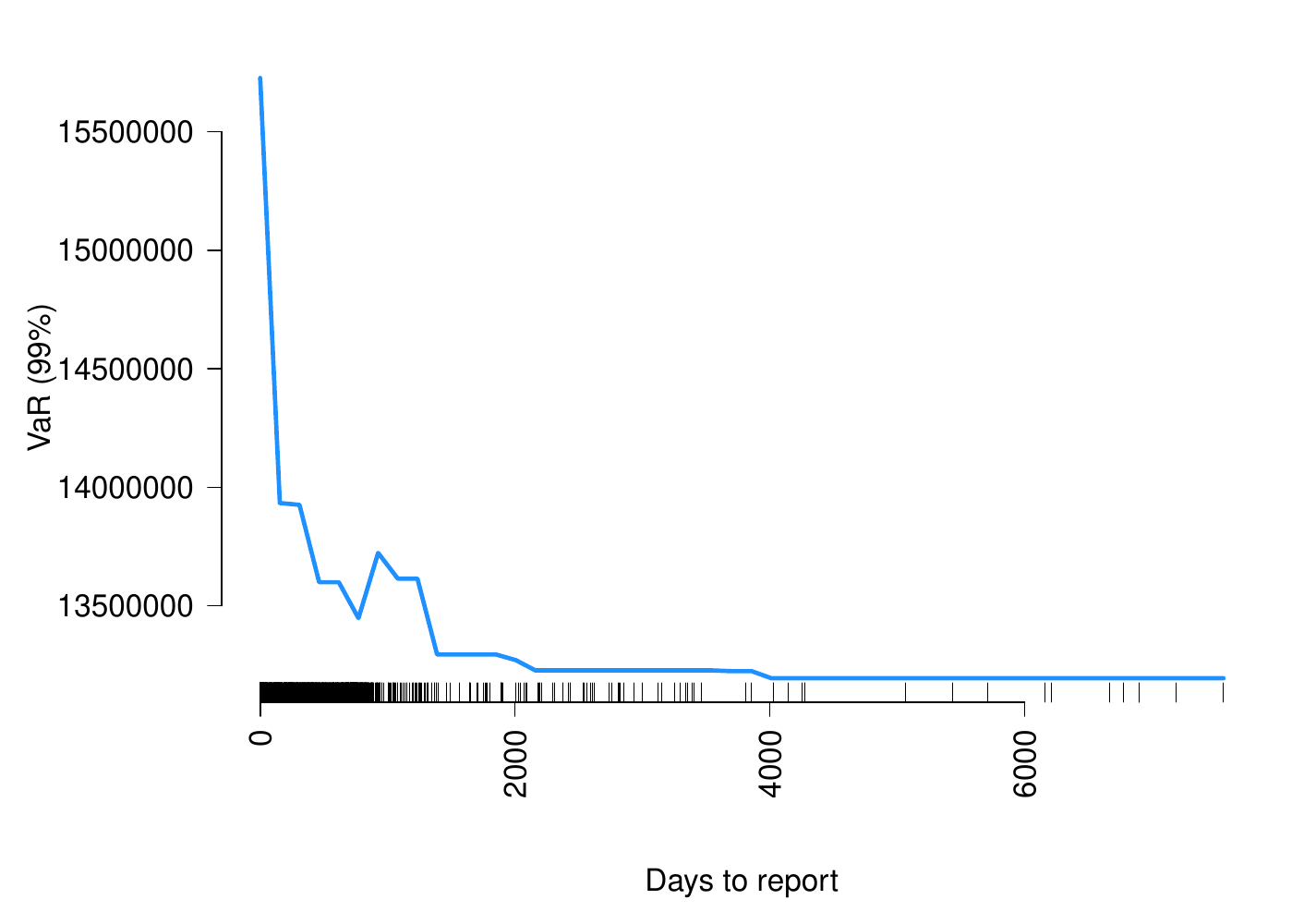}
    \includegraphics[width=0.49\linewidth]{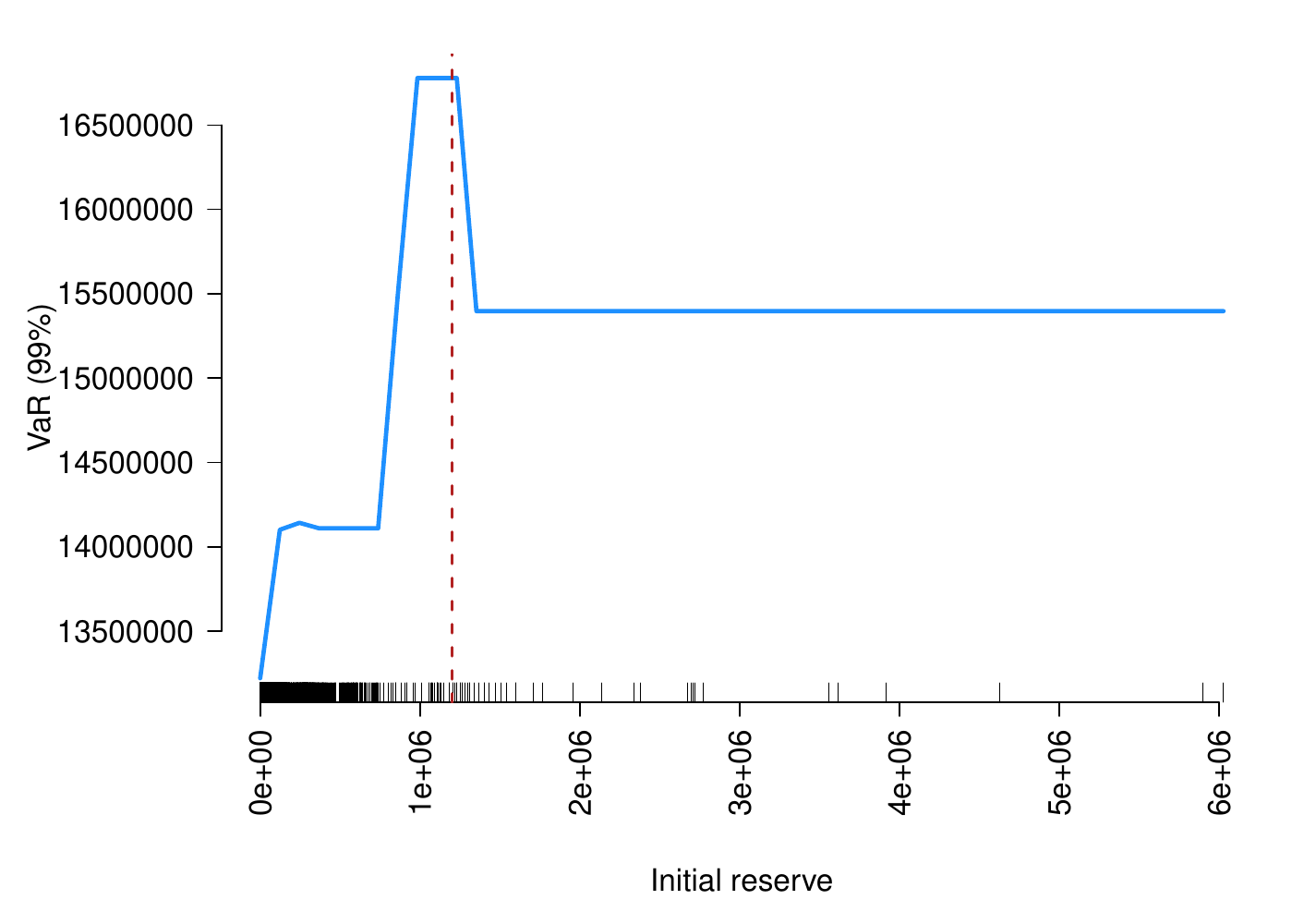}
    \includegraphics[width=0.49\linewidth]{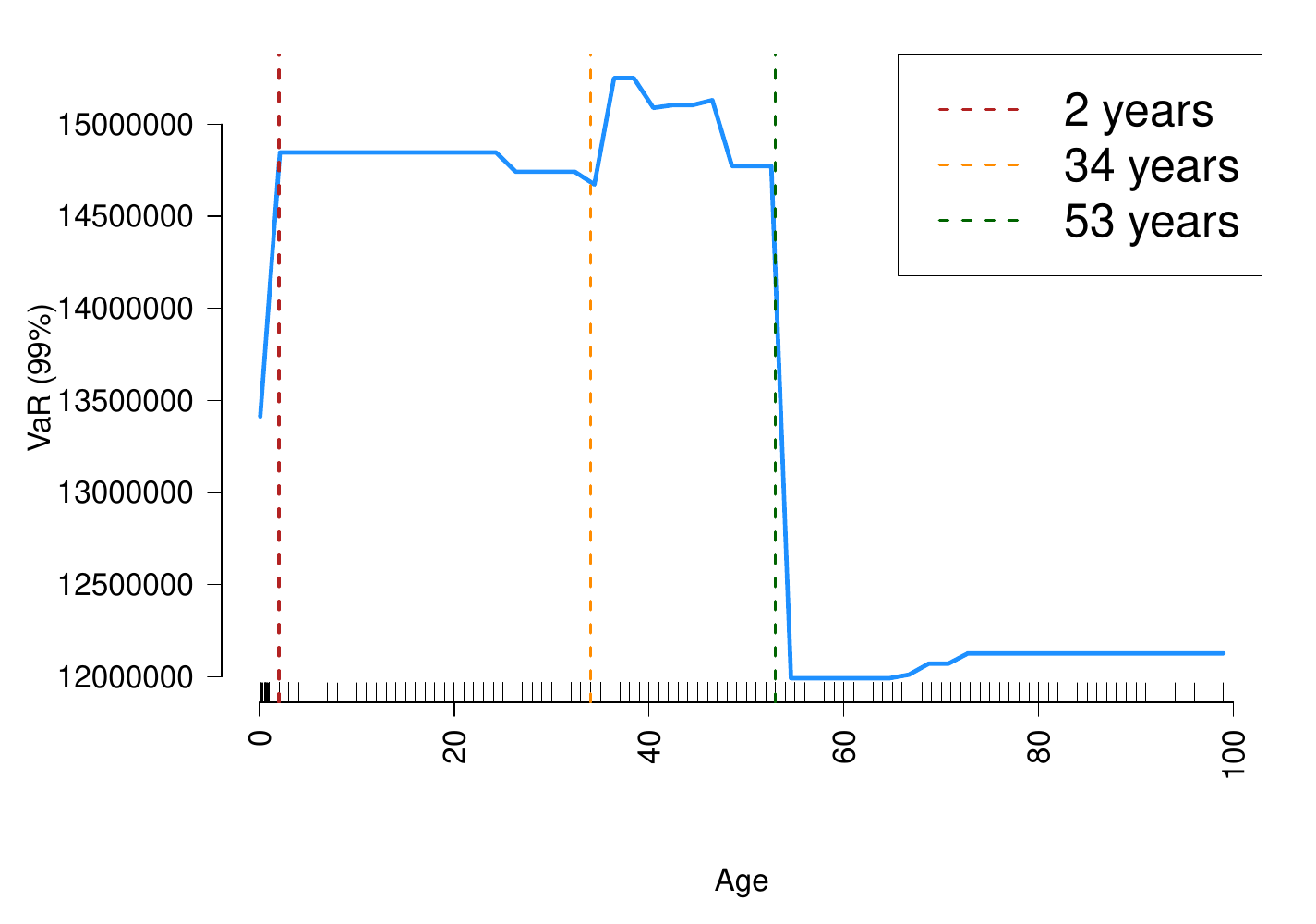}
    \includegraphics[width=0.49\linewidth]{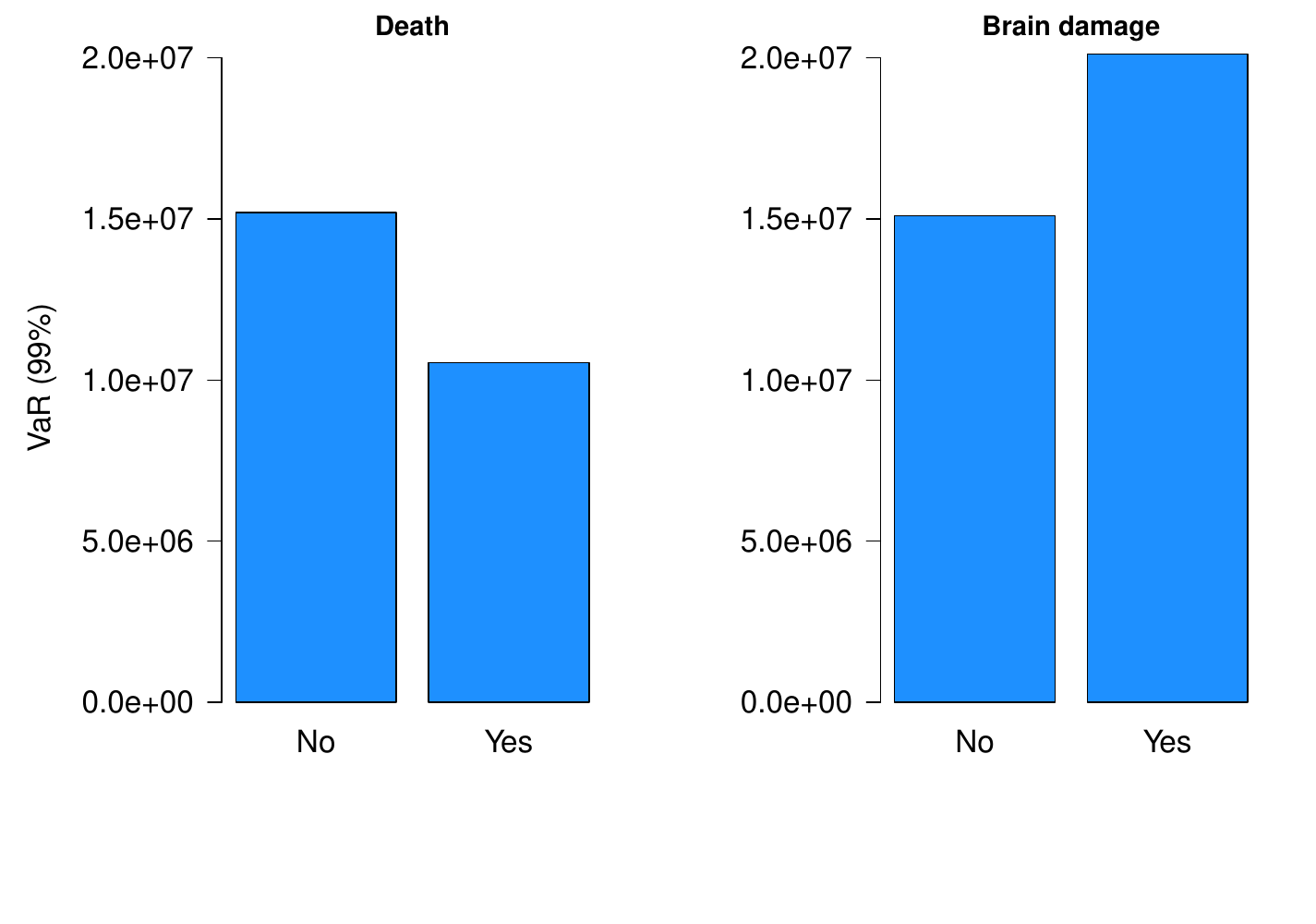}
    \caption{Partial dependence plots}
    \label{fig:pd-plots}
\end{figure}

Several interesting observations can be made. The relationship between the time to settlement and the tail of the settlement cost is non-monotonic. A possible explanation could be that cases settled within the first three years tend to involve minor injuries with clear liability, resulting in moderate and predictable amounts. The drop observed around three years likely reflects a selection effect, whereby cases settled just before trial are subject to compromise under pressure from both parties seeking to avoid litigation costs (consistent with \cite{Farber}). The subsequent increase between five and ten years captures more complex and contested cases, where severe sequelae, difficult causality, and accumulating legal costs drive up the final amount. Finally, cases exceeding ten years correspond to catastrophic injuries which are rare but extremely costly. 
The relationship with the number of days before reporting to the insurer is clearly decreasing, indicating that the riskiest claims are typically reported very quickly, while claims whose reporting tends to be delayed may correspond to less severe injuries or situations where liability is less clearly established, resulting in lower settlement costs.
The initial reserve is essentially positively related to the riskiness of the claim showing a good understanding of the risk by the experts. Above $1.2$ millions US dollars we observe a slight decrease but this is not really significant as less than $3\%$ of the excesses used to fit the gradient boosting model exhibit such a high reserve.
The relationship between the victim's age at the time of injury and the 99\% VaR follows an economically intuitive pattern. The risk level increases linearly up to age 3, then stabilizes at an elevated level until age 34, reflecting a long remaining life expectancy over which future losses accumulate. It further increases between ages 34 and 53, where lost earnings are at their highest, before decreasing beyond age 53 as life expectancy shortens and future income losses diminish.
Regarding the nature of the injury, as expected, death results in a bounded settlement distribution with a lower tail risk, whereas severe injuries such as brain damage are associated with substantially higher settlement costs, reflecting the long-term care and income loss that such conditions entail.

\section{Conclusion}
\label{sec:ccl}

We provide empirical risk minimization theory for the gradient boosting algorithm applied to estimating a covariate-dependent GP model. In particular, we discuss multiple kinds of errors associated with this procedure, notably the bias arising from the fact that the GP distribution is merely an asymptotic model for excesses above a high threshold, and we show how this error can be controlled under second-order regular variation. Our bounds also make explicit the other bias–variance trade-off associated with the choice of the number of boosting iterates $T$. Unlike \citet{biau2021optimization}, we are not able to provide an algorithmic analysis of the procedure as $T \to \infty$ since, even in the new parametrization, the GP loss is not globally convex. However, we show that statistical guarantees may be obtained under less restrictive conditions.

Understanding the tail of claim distributions (conditionally on covariate information) is key for the insurance industry in order to properly assess risk. We illustrate the usefulness of the gradient boosting method applied to the GP model in this context using medical liability insurance data provided by the Texas Department of Insurance. The results exhibit a good fit and shed light on the relationship between covariates and the heaviness of the claim distribution. In particular, we illustrate once again that the number of days to settlement is a key variable for understanding the claim size distribution, and that its estimation is of primary importance for the insurer \citep{lopez2019censored}.
Many other regression algorithms are available for the problem we study---generalized additive models \citep{chavez2016gam}, random forests \citep{gnecco2024ERF}, and neural networks \citep{pasche2024neural}, among others---and, depending on their needs, risk professionals should compare the performance of or combine several of these models to achieve the accuracy they seek.

\bibliography{references.bib}

\begin{thebibliography}{}

\bibitem[Balkema and De~Haan, 1974]{balkema1974residual}
Balkema, A.~A. and De~Haan, L. (1974).
\newblock Residual life time at great age.
\newblock {\em The Annals of probability}, pages 792--804.

\bibitem[Beirlant and Goegebeur, 2004]{beirlant2004local}
Beirlant, J. and Goegebeur, Y. (2004).
\newblock Local polynomial maximum likelihood estimation for pareto-type
  distributions.
\newblock {\em Journal of Multivariate Analysis}, 89(1):97--118.

\bibitem[Beirlant et~al., 2006]{beirlant2006statistics}
Beirlant, J., Goegebeur, Y., Segers, J., and Teugels, J.~L. (2006).
\newblock {\em Statistics of extremes: theory and applications}.
\newblock John Wiley \& Sons.

\bibitem[Biau and Cadre, 2021]{biau2021optimization}
Biau, G. and Cadre, B. (2021).
\newblock Optimization by gradient boosting.
\newblock In {\em Advances in contemporary statistics and econometrics.
  Festschrift in honor of Christine Thomas-Agnan}, pages 23--44. Cham:
  Springer.

\bibitem[Borak et~al., 2011]{borak2011models}
Borak, S., Misiorek, A., and Weron, R. (2011).
\newblock Models for heavy-tailed asset returns.
\newblock In {\em Statistical tools for finance and insurance}, pages 21--55.
  Springer.

\bibitem[Boucheron et~al., 2013]{boucheron2013concentration}
Boucheron, S., Lugosi, G., and Massart, P. (2013).
\newblock {\em Concentration inequalities. {A} nonasymptotic theory of
  independence}.
\newblock Oxford: Oxford University Press.

\bibitem[Chavez-Demoulin and Davison, 2005]{chavez2005generalized}
Chavez-Demoulin, V. and Davison, A.~C. (2005).
\newblock Generalized additive modelling of sample extremes.
\newblock {\em Journal of the Royal Statistical Society Series C: Applied
  Statistics}, 54(1):207--222.

\bibitem[Chavez-Demoulin et~al., 2016]{chavez2016gam}
Chavez-Demoulin, V., Embrechts, P., and Hofert, M. (2016).
\newblock An extreme value approach for modeling operational risk losses
  depending on covariates.
\newblock {\em The Journal of Risk and Insurance}, 83(3):735--776.

\bibitem[Cox and Reid, 1987]{cox1987parameter}
Cox, D.~R. and Reid, N. (1987).
\newblock Parameter orthogonality and approximate conditional inference.
\newblock {\em Journal of the Royal Statistical Society: Series B
  (Methodological)}, 49(1):1--18.

\bibitem[Davison and Smith, 1990]{davison1990models}
Davison, A.~C. and Smith, R.~L. (1990).
\newblock Models for exceedances over high thresholds.
\newblock {\em Journal of the Royal Statistical Society Series B: Statistical
  Methodology}, 52(3):393--425.

\bibitem[de~Haan and Ferreira, 2006]{deHaanF2006EVT}
de~Haan, L. and Ferreira, A. (2006).
\newblock {\em Extreme value theory. {An} introduction.}
\newblock Springer Ser. Oper. Res. Financ. Eng. New York, NY: Springer.

\bibitem[Denuit et~al., 2019]{denuit2019effective}
Denuit, M., Hainaut, D., and Trufin, J. (2019).
\newblock {\em Effective statistical learning methods for actuaries {I}. {GLMs}
  and extensions}.
\newblock Springer Actuar. Cham: Springer.

\bibitem[Einmahl and Mason, 2005]{einmahl2005uniform}
Einmahl, U. and Mason, D.~M. (2005).
\newblock Uniform in bandwidth consistency of kernel-type function estimators.
\newblock {\em Ann. Stat.}, 33(3):1380--1403.

\bibitem[Embrechts et~al., 2013]{embrechts2013modelling}
Embrechts, P., Kl{\"u}ppelberg, C., and Mikosch, T. (2013).
\newblock {\em Modelling extremal events: for insurance and finance},
  volume~33.
\newblock Springer Science \& Business Media.

\bibitem[Farber and White, 1991]{Farber}
Farber, H. and White, M. (1991).
\newblock Medical malpractice: an empirical examination of the litigation
  process.
\newblock {\em The Rand journal of economics}, 22(2):199--217.

\bibitem[Farkas et~al., 2024]{farkas2024gptree}
Farkas, S., Heranval, A., Lopez, O., and Thomas, M. (2024).
\newblock Generalized {Pareto} regression trees for extreme event analysis.
\newblock {\em Extremes}, 27(3):437--477.

\bibitem[Farkas et~al., 2021]{farkas2021cyber}
Farkas, S., Lopez, O., and Thomas, M. (2021).
\newblock Cyber claim analysis using generalized pareto regression trees with
  applications to insurance.
\newblock {\em Insurance: Mathematics and Economics}, 98:92--105.

\bibitem[Friedman, 2001]{friedman2001greedy}
Friedman, J.~H. (2001).
\newblock Greedy function approximation: {A} gradient boosting machine.
\newblock {\em Ann. Stat.}, 29(5):1189--1232.

\bibitem[Gimeno et~al., 2022]{gimeno2022extreme}
Gimeno, L., Sor{\'\i}, R., Vazquez, M., Stojanovic, M., Algarra, I.,
  Eiras-Barca, J., Gimeno-Sotelo, L., and Nieto, R. (2022).
\newblock Extreme precipitation events.
\newblock {\em Wiley Interdisciplinary Reviews: Water}, 9(6):e1611.

\bibitem[Gnecco et~al., 2024]{gnecco2024ERF}
Gnecco, N., Terefe, E.~M., and Engelke, S. (2024).
\newblock Extremal random forests.
\newblock {\em J. Am. Stat. Assoc.}, 119(548):3059--3072.

\bibitem[Goldie and Smith, 1987]{goldie1987slow}
Goldie, C.~M. and Smith, R.~L. (1987).
\newblock Slow variation with remainder: Theory and applications.
\newblock {\em Quarterly Journal of Mathematics}, 38:45--71.

\bibitem[Hill, 1975]{hill1975simple}
Hill, B.~M. (1975).
\newblock A simple general approach to inference about the tail of a
  distribution.
\newblock {\em The annals of statistics}, pages 1163--1174.

\bibitem[Leboeuf et~al., 2020]{leboeuf2020trees}
Leboeuf, J.-S., LeBlanc, F., and Marchand, M. (2020).
\newblock Decision trees as partitioning machines to characterize their
  generalization properties.
\newblock In Larochelle, H., Ranzato, M., Hadsell, R., Balcan, M., and Lin, H.,
  editors, {\em Advances in Neural Information Processing Systems}, volume~33,
  pages 18135--18145. Curran Associates, Inc.

\bibitem[Lhaut, 2021]{lhaut2021thesis}
Lhaut, S. (2021).
\newblock Inégalités de concentration pour évènements rares.
\newblock Master's thesis, UCLouvain.

\bibitem[Lhaut et~al., 2022]{lhaut2022uniform}
Lhaut, S., Sabourin, A., and Segers, J. (2022).
\newblock Uniform concentration bounds for frequencies of rare events.
\newblock {\em Stat. Probab. Lett.}, 189:7.
\newblock Id/No 109610.

\bibitem[Lopez, 2019]{lopez2019censored}
Lopez, O. (2019).
\newblock A censored copula model for micro-level claim reserving.
\newblock {\em Insurance: Mathematics and Economics}, 87:1--14.

\bibitem[Maillart, 2021]{maillart2021}
Maillart, A. (2021).
\newblock {\em Quelques méthodes d’explicabilité pour les modèles
  d’apprentissage statistique en actuariat}.
\newblock PhD thesis, Université Claude Bernard Lyon 1.

\bibitem[Maillart and Sornette, 2010]{maillart2010heavy}
Maillart, T. and Sornette, D. (2010).
\newblock Heavy-tailed distribution of cyber-risks.
\newblock {\em The European Physical Journal B-Condensed Matter and Complex
  Systems}, 75(3):357--365.

\bibitem[Mason et~al., 1999]{mason1999boosting}
Mason, L., Baxter, J., Bartlett, P., and Frean, M. (1999).
\newblock Boosting algorithms as gradient descent.
\newblock In Solla, S., Leen, T., and M\"{u}ller, K., editors, {\em Advances in
  Neural Information Processing Systems}, volume~12. MIT Press.

\bibitem[Mason et~al., 2000]{mason2000functional}
Mason, L., Baxter, J., Bartlett, P., and Frean, M. (2000).
\newblock Functional gradient techniques for combining hypotheses.
\newblock {\em Advances in Large Margin Classifiers}.

\bibitem[McDiarmid, 1998]{mcdiarmid1998concentration}
McDiarmid, C. (1998).
\newblock Concentration.
\newblock In {\em Probabilistic methods for algorithmic discrete mathematics},
  pages 195--248. Springer.

\bibitem[Mohri et~al., 2018]{mohri2018foundations}
Mohri, M., Rostamizadeh, A., and Talwalkar, A. (2018).
\newblock {\em Foundations of machine learning}.
\newblock Adapt. Comput. Mach. Learn. Cambridge, MA: MIT Press, 2nd edition
  edition.

\bibitem[Moins et~al., 2023]{moins2023reparam}
Moins, T., Arbel, J., Girard, S., and Dutfoy, A. (2023).
\newblock Reparameterization of extreme value framework for improved bayesian
  workflow.
\newblock {\em Computational Statistics \& Data Analysis}, 187:107807.

\bibitem[Pasche and Engelke, 2024]{pasche2024neural}
Pasche, O.~C. and Engelke, S. (2024).
\newblock Neural networks for extreme quantile regression with an application
  to forecasting of flood risk.
\newblock {\em The Annals of Applied Statistics}, 18(4):2818--2839.

\bibitem[Pickands~III, 1975]{pickands1975statistical}
Pickands~III, J. (1975).
\newblock Statistical inference using extreme order statistics.
\newblock {\em the Annals of Statistics}, pages 119--131.

\bibitem[Resnick, 1987]{resnick1987extreme}
Resnick, S.~I. (1987).
\newblock {\em Extreme values, regular variation and point processes}.
\newblock Springer New York, NY.

\bibitem[Resnick, 2007]{resnick2007heavy}
Resnick, S.~I. (2007).
\newblock {\em Heavy-tail phenomena: probabilistic and statistical modeling}.
\newblock Springer New York, NY.

\bibitem[van~der Vaart and Wellner, 2023]{vdVW1996weak}
van~der Vaart, A.~W. and Wellner, J.~A. (2023).
\newblock {\em Weak convergence and empirical processes. {With} applications to
  statistics}.
\newblock Springer Ser. Stat. Cham: Springer, 2nd expanded edition edition.

\bibitem[Vapnik, 1991]{vapnik1991principles}
Vapnik, V. (1991).
\newblock Principles of risk minimization for learning theory.
\newblock {\em Advances in neural information processing systems}, 4.

\bibitem[Velthoen et~al., 2023]{velthoen2023gradient}
Velthoen, J., Dombry, C., Cai, J.-J., and Engelke, S. (2023).
\newblock Gradient boosting for extreme quantile regression.
\newblock {\em Extremes}, 26(4):639--667.

\bibitem[Wang and Li, 2013]{wang2013estimation}
Wang, H.~J. and Li, D. (2013).
\newblock Estimation of extreme conditional quantiles through power
  transformation.
\newblock {\em Journal of the American Statistical Association},
  108(503):1062--1074.

\end{thebibliography}

\appendix

\section{Notation}
\label{app:notation}

The following notation is used through the paper.
\begin{itemize}
	\item $\R$ – set of real numbers ; $\R_+ = (0,\infty)$ - set of strictly positive real numbers ; $\N = \{0,1,2,\ldots\}$ - set of non-negative integers ; $\N_0 = \{1,2,3,\ldots\}$ - set of strictly positive integers.
	\item $\borel(\R^d)$ – Borel $\sigma$-algebra on $\R^d$ for $d \in \N_0$.
	\item $[a:b] = \{n \in \N: a \leq n \leq b\}$  - set of integers between $a,b \in \N$. We set $[d] = [1:d]$ for any $d \in \N_0$.
	\item $(\Omega,\cA,\P)$ – common probability space on which every random variable is defined.
	\item $P_{\X} = \P \circ \X^{-1}$ - law of  a random vector $\X : (\Omega,\cA) \to (\R^d, \borel(\R^d))$.
	\item $P_{Y \mid \X}$ – conditional probability distribution (in the sense of Markov kernels)  of a real-valued random variable $Y : (\Omega,\cA) \to (\R, \borel(\R))$ given a random vector $\X : (\Omega,\cA) \to (\R^d, \borel(\R^d)$. For fixed $\x \in \R^d$, we denote by $\P_{Y \mid \X = \x}(\cdot)$ the probability measure  $P_{Y \mid \X}(\cdot,\x)$. Informally, we will sometimes use the notation $Y \mid \X = \x$ to denote a random variable whose law equals $\P_{Y \mid \X = \x}(\cdot)$.
	\item $\E[\X] = \lp \int_{\Omega} X^{(j)}(\omega) \; \diff\P(\omega) = \int_{\R} x \diff\P_{X^{(j)}}(x): j \in [d] \rp$ – expectation of a random vector $\X : (\Omega,\cA) \to (\R^d, \borel(\R^d))$ with $d \in \N_0$. Similarly, given a random variable $Y : (\Omega,\cA) \to (\R, \borel(\R))$, we set $\E[Y \mid \X = \x] = \int_{\R} y \diff\P_{Y \mid \X = \x}(y)$.
	\item $\bF_{Y}(y) = \P(Y > y)$ – survival function of a random variable $Y : (\Omega,\cA) \to (\R, \borel(\R))$ evaluated at $y \in \R$. Similarly, given a random vector $\X : (\Omega,\cA) \to (\R^d, \borel(\R^d))$, we set $\bF_{Y \mid \X = \x}(y) = \P_{Y \mid \X = \x}((y,\infty))$. Informally, we will also write $\bF_{Y \mid \X = \x}(y) = \P(Y > y \mid \X = \x)$. We will also use the notation $F_Y(y) = 1 - \bF_{Y}(y)$ for the cumulative distribution function of $Y$ evaluated at $y \in \R$ (and the conditional versions).
	\item $\RV_\rho$ for $\rho \in \R$ – set of regularly varying functions at infinity with tail index $\rho$, i.e.,
	\[
		f \in \RV_\rho \iff f : \R_+ \to \R_+ \text{ satisfies }  \lim_{t \to \infty} \frac{f(tx)}{f(t)} = x^{\rho}, \qquad \forall x > 0.
	\]
	A function in $\RV_0$ is called slowly varying.
	\item $f^{-1}(y) = \inf\{x \in \R: f(x) \geq y\}$ – generalized inverse of a non-decreasing function $f : \R \to \R$.
	\item i.i.d. – independent and identically distributed.
	\item $\Be(p)$ – Bernoulli distribution with parameter $p \in [0,1]$ ; $\Bin(n,p)$ – Binomial distribution with parameter $n \in \N$ and $p \in [0,1]$.
	\item $\I\{\event\}$ – indicator random variable of the event $\event \in \cA$, i.e.,
	\[
		\I \{\event\}(\omega)
		=
		\begin{cases}
			1 &\text{if } \omega \in \event \\
			0 &\text{else}
		\end{cases}
		\sim \Be(\P(\event)).
	\]
	\item $M_{Y}(\lambda) = \E \lc \exp\{\lambda Y\} \rc$ – moment generating function a random variable $Y : (\Omega,\cA) \to (\R, \borel(\R))$ evaluated at $\lambda \in \R$.
	\item $\|\x\|_\infty = \max \lacc |x^{(j)}|: j \in [d] \racc$ – $L_\infty$ norm of a vector $\x \in \R^d$.
	\item $A \triangle B = \lp A \cap (E \setminus B) \rp \bigcup \lp (E \setminus A) \cap B \rp$ – symmetric difference of two sets $A,B \subset E$.
	\item $\mu(f) = \int_{\mathbb{Z}} f(z) \diff\mu(z)$ – integral of a function $f : \mathbb{Z} \to \R$ with respect to a measure $\mu$ on a measure space $(\mathbb{Z}, \mathcal{Z})$.
\end{itemize}

\section{Auxiliary results}
\label{app:auxiliary}

In the next three results, we will investigate the behavior of $\Dot{L}_j(\bphi,z)$ for $j \in [2]$, the partial derivative of the map $L$ with respect to its $j$-th component of its bivariate parameter $\bphi = (\nu,\xi)$ evaluated at $(\bphi,z)$. It follows from the chain rule that if $\mathcal{L}$ denote the negative GP log-likelihood in the usual $\btheta = (\sigma, \xi)$ parametrization, we have
\begin{align*}
    \nabla_{\bphi} L(\bphi,z)^\tr =  
    \Bigg(
        &\Dot{L}_1(\bphi,z) = \frac{1}{\xi + 1} \Dot{\mathcal{L}}_1 \lp (\frac{\nu}{\xi + 1}, \xi), z \rp, \\
        &\Dot{L}_2(\bphi,z) = - \frac{\nu}{(\xi + 1)^2} \Dot{\mathcal{L}}_1 \lp (\frac{\nu}{\xi + 1}, \xi), z \rp + \Dot{\mathcal{L}}_2 \lp (\frac{\nu}{\xi + 1}, \xi), z \rp
    \Bigg)
\end{align*}
A small computation based on the formulas for the derivatives $\Dot{\mathcal{L}}_j$ in Appendix A of \citet{velthoen2023gradient} leads to
\[
    \Dot{L}_1(\bphi,z)
    = \frac{1}{\nu} \lacc 1 - \frac{(\xi+1)^2 z}{\nu + \xi(\xi+1)z}\racc
\]
and
\[
    \Dot{L}_2(\bphi,z)
    = \lp 1 + \frac{1}{\xi} \rp \frac{(1+2\xi)z}{\nu + \xi(\xi+1) z} - \frac{1}{\xi^2} \log \lp 1 + \frac{\xi(\xi+1)}{\nu}z \rp - \frac{1}{\xi + 1}  
\]

The following lemma corresponds to Lemma 11 in \citet{farkas2024gptree} for the new parameter.
\begin{lemma}[Sub-exponential bound]
	\label{lem:upper-fun}
	For any $\bphi \in \reparamspace$ and $y > u > 0$, we have that the quantities
	\[
		\labs \Dot{L}_j(\bphi, y-u) \rabs \leq \Phi(y) = \cst [1 + \log(1 + w y)],, \qquad j \in [2],
	\]
	where $\cst > 0$ is a constant and $w = \xi_U(\xi_U+1)/\nu_L$. 
    Furthermore, the random variable $\Phi(Y)$ is sub-exponential in the sense that there exists $\rho > 0$ such that
	\[
		M_{\Phi(Y)}(\lambda)  < \infty, \qquad
		\forall |\lambda| \leq \rho.
	\]
\end{lemma}

The following lemma is an adaptation of the bounds in the proof of Lemma 10 in \citet{farkas2024gptree} for the new parametrization.

\begin{lemma}[Lipschitz cost functions with respect to the parameter]
	\label{lem:lip-param}
	Let $\bphi,\bphi' \in \reparamspace$ and $y>u>0$. For any $M>0$, we have for $j \in [2]$,
	\[
		\labs
		\Dot{L}_j(\bphi,y-u)
		-
		\Dot{L}_j(\bphi',y-u)
		\rabs \I\{\Phi(y) \leq M\}
		\leq
		\mu_j \left\| \bphi - \bphi' \right\|_\infty
	\]
	where $\mu_1 = \mu_1(M)$ and $\mu_2 = \mu_2(M)$ are defined by
	\begin{align*}
		\mu_1 &= \frac{M}{\nu_L^3} + \frac{\cst_1}{\nu_L} \\
		\mu_2 &= \frac{1}{1+\xi_L} + \frac{2M}{\xi_L^3} + \frac{\cst_2}{\nu_L\xi_L^2} \lp \frac{\xi_U(1+\xi_U)}{\nu_L} + 1 + 2\xi_U \rp + \frac{1+2\xi_U}{\xi_L^2} + \lp 1 + \frac{1}{\xi_L} \rp \cst_3
	\end{align*}
	for some constants $\cst_1,\cst_2, \cst_3 >0$ not depending on $n$.
\end{lemma}

\begin{lemma}[Difference with respect to the excess]
	\label{lem:lip-z}
	Let $\bphi \in \reparamspace$ and $z_1,z_2>0$. We have
	\[
		\labs
		\Dot{L}_1(\bphi,z_1)
		-
		\Dot{L}_1(\bphi,z_2)
		\rabs
		\leq 
         \frac{\lp 1 + \xi_U \rp^3 |z_1-z_2|}{(\nu_L + (1+\xi_U)\xi_L z_1) (\nu_L + (1+\xi_U)\xi_L z_2))}
	\]
	and
    \begin{align*}
    	\labs
    		\Dot{L}_2(\bphi,z_1)
    		-
    		\Dot{L}_2(\bphi,z_2)
		\rabs
		&\leq \frac{\xi_U(1+\xi_U)}{\xi_L^2 \nu_L} \frac{|z_1-z_2|}{1 + \min\{z_1,z_2\}} \\
        &\qquad + \lp 1 + \frac{1}{\xi_L} \rp \frac{(1+2\xi_U)(2+\xi_U) |z_1-z_2|}{\nu_L + \xi_L(1+\xi_L) \min\{z_1,z_2\}}.
    \end{align*}
	In particular, if $|\tfrac{z_1}{z_2}-1| \leq \eps$, we get that for some constant $\kappa > 0$ not depending on $n$,
	\[
		\left\|
			\nabla_{\bphi} \lc L(\bphi,z_1) - L(\bphi,z_2) \rc 
		\right\|_\infty
		\leq
		\kappa \, \eps. 
	\]
\end{lemma}

\begin{proof}
	The first part of the lemma is a direct calculation based on the formula for the derivatives of the loss functions which may be adapted from Appendix A in \citet{velthoen2023gradient}. The second part is a direct computation.
\end{proof}

\begin{lemma}[Tail bound for binomial random variables]
\label{lem:tail-bound-bin}
	Let $S_n = \sum_{i=1}^{n} Y_i$ where $\{Y_i : i\in  [n]\}$ is a collection of i.i.d. random variables distributed as $\Be(p)$ for some $p \in (0,1)$. Then, for any $\delta \in (0,1)$, there exists an event with probability at least $1-\delta$ on which,
	\[
		S_n \leq \ee np + \sqrt{\frac{6\ee}{5} np \log(1/\delta)}.
	\]
\end{lemma}

\begin{proof}
	For any $\eps>0$ and $\lambda > 0$, we observe that by Markov's inequality
	\begin{align*}
		\P(S_n > n\eps)
		&=
		\P \lp \exp(\lambda S_n) > \exp(n\lambda\eps) \rp \\
		&\leq
		\exp(-n\lambda\eps) \lp \E \lc \exp \lp \lambda Y_1 \rp \rc \rp^n \\
		&\leq \exp(-n\lambda\eps) \exp \lp \mathrm{e}^\lambda np \rp \\
		&=\exp \lacc - n \lp \lambda \eps - \mathrm{e}^\lambda p \rp \racc.
	\end{align*}
	The map
	\[
	\lambda > 0 \mapsto \lambda \eps - \mathrm{e}^\lambda p
	\]
	has a unique maximum at
	\[
	\lambda = \log \lp \frac{\eps}{p} \rp
	\]
	leading to the fact that for any $\eps >0$,
	\[
	\P(S_n > n\eps)
	\leq
	\exp \lacc - n\eps \lp \log \lp \frac{\eps}{p} \rp - 1 \rp \racc.
	\]
	The bound is only useful if $\eps > \ee p$. Therefore, we take $\beta >0$ and consider
	\[
	\eps(\beta) = (1 + \beta) \ee p.
	\]
	For any $\beta > 0$, we thus have
	\begin{equation}
		\label{eq:bound_psum}
		\P(S_n > n\eps(\beta))
		\leq
		\exp \lacc - np\ee(1+\beta)\log(1+\beta) \racc.
	\end{equation}
	It follows from Exercise 2.8 in \citet{boucheron2013concentration} that
	\[
	(1+\beta)\log(1+\beta)
	\geq
	\frac{\beta^2}{2(1+\beta/3)} + \beta = g(\beta).
	\]
	The function $\beta > 0 \mapsto g(\beta) > 0$ is invertible since it is strictly increasing and for any $y>0$, the value $\beta(y) > 0$ which is such that $g(\beta(y)) = y$ satisfies $\beta(y) \leq \sqrt{6y/5}$. Therefore, by monotonicity, we get $g(\sqrt{6y/5}) \geq y$ for any $y > 0$. As such, we get that the upper bound in \eqref{eq:bound_psum} is smaller than $\delta \in (0,1)$ if we take
	\[
	\beta \geq \beta(\delta)
	= \sqrt{\frac{6}{5\ee} \frac{1}{np} \log(1/\delta)}.
	\]
	The associated value of $\eps(\beta)$ is given by
	\[
	\eps(\delta)
	= \ee p + \sqrt{\frac{6\ee}{5} \frac{p}{n} \log(1/\delta)}.
	\]
	This means that there exists with probability at least $1-\delta$ on which
	\[
	S_n \leq \ee np + \sqrt{\frac{6\ee}{5} np \log(1/\delta)}. \qedhere
	\]
\end{proof}

The following lemma is a simple extension of Lemma 2.1 in \citet{lhaut2022uniform}.

\begin{lemma}[Conditioning trick]
	\label{lem:CT}
	Let $Z_1,\ldots,Z_n$ be an independent random sample from $P$ on an arbitrary space $\cZ$ and let $P_n$ denote the associated empirical measure on $\cZ$. Let $\cG$ be a class of real-valued functions on $\cZ$. Let $E \subset \cZ$ be such that $p = P(E) \in (0,1)$. Let $K = \sum_{i=1}^n \I \{Z_i \in E\} \sim \Bin(n,p)$ denote the random number of sample points in $E$. Then, we have for any $k \in [0:n]$,
	\[
	\lc (P_n(g\I_E))_{g \in \cG} \mid K = k \rc
	\eqlaw
	\lc \frac{k}{n} \lp P_k^Y(g) \rp_{g \in \cG} \rc,
	\]
	where $P_k^Y$ denotes the empirical measure of independent random variables $Y_1,\ldots,Y_k$ distributed as $P_E(\cdot) = P(\cdot \cap E) / p$. Furthermore, if the class $\cG$ is pointwise measurable, we also have
	\[
	\lc \sup_{g \in \cG} P_n(g\I_E)) \mid K = k \rc
	\eqlaw
	\lc \frac{k}{n} \sup_{g \in \cG}  P_k^Y(g) \rc
	\]
\end{lemma}

\begin{proof}
	The proof is identical to the one in \citet{lhaut2021thesis} for the case where $\cG$ consists of indicator functions.
\end{proof}

The following result will be our main concentration bound \citep[Theorem 3.8]{mcdiarmid1998concentration}. This corresponds to an extension of the classical Bernstein's inequality to general functions of independent random variables.

\begin{theorem}
\label{thm:concentration-expectation}
	Let $X_1,\ldots,X_n$ be independent random variables taking values in a space $\mathcal{X}$. Let $f : \mathcal{X}^n \to \R$ and $Z = f(X_1,\ldots,X_n)$. For $i \in [n]$ and $x_1,\ldots,x_i \in \mathcal{X}$, let
	\[
		g_i(x_1,\ldots,x_i) =
		\E[Z \mid X_1 = x_1,\ldots,X_i = x_i] - 
		\E[Z \mid X_1 = x_1,\ldots,X_{i-1} = x_{i-1}].
	\]
	If the maximal deviation
	\[
	c = \max_{i \in [n]} \sup_{x_1,\ldots,x_i \in \mathcal{X}} g_i(x_1,\ldots,x_i)
	\]
	and the supremum of the sum of the conditional variances
	\[
	\sigma^2 =
	\sum_{i=1}^n \sup_{x_1,\ldots,x_{i-1} \in \mathcal{X}} \Var[g_i(X_1,\ldots,X_i) \mid X_1 = x_1, \ldots, X_{i-1} = x_{i-1}]
	\] 
	are finite, then for any $t>0$,
	\[
	\P(Z \geq \E[Z] + t) \leq
	\exp \lp \frac{-t^2}{2(\sigma^2 + ct/3)} \rp.
	\]
	Equivalently, given $\delta \in (0,1)$, there exists an event $\event \in \cA$ such that $\P(\event) \geq 1-\delta$ on which
	\[
		Z \leq \E[Z] + \sqrt{2\sigma^2\log(1/\delta)} + \frac{2c}{3} \log(1/\delta).
	\]
\end{theorem}

A situation that will be of particular interest to us is the case where
\begin{equation}
\label{eq:Z=supp}
    Z = \sup_{g \in \cG}
    \labs
        P_n(g \I_E) - P(g \I_E)
    \rabs
\end{equation}
with $E \subset \mathcal{X}$ satisfying $P(E) \leq p$ for some (small) $p \in (0,1)$ and $\sup_{g \in \cG} \|g\|_\infty \leq M$ for some $M>0$. In this case, we have
\[
    c \leq \frac{2M}{n} \qquad \text{and} \qquad \sigma^2 \leq \frac{4M^2p}{n}.
\]

Next, we will need small results on Rademacher complexities. Recall that given an i.i.d. sample $Z_1, \dots, Z_n$ on $\spaceZ$ and $\cG$ a class of functions $g : \spaceZ \to \R$, the (expected) Rademacher complexity of the class $\cG$ of size $n$ is given by
\[
    \RCData(\cG) =  \E \lc \sup_{g \in \cG} \frac{1}{n} \sum_{i=1}^n \eps_i g(Z_i) \rc
\]
where $(\eps_i : i \in [n])$ is a set of i.i.d. Rademacher random variables, i.e., $\P(\eps_i = -1) = \P(\eps_i = 1) = 0.5$. Recall also that the link between $\RCData(\cG)$ and $\E[Z]$ with $Z$ as in \eqref{eq:Z=supp} is provided by Lemma 2.3.1 in \citet{vdVW1996weak} which ensures that
\[
    \E[Z] \leq 2 \RCData(\cG \I_E).
\]

\begin{lemma}[Rademacher Complexity of a union]
	\label{lem:RC-union}
	Let $\mathcal{F}$ and $\mathcal{G}$ be two classes of real-valued functions defined on the same space $\spaceZ$.
	Then, for any i.i.d. sample $Z_1, \dots, Z_n$ on $\spaceZ$, we have
	\[
	\RCData(\mathcal{F} \cup \mathcal{G})
	\;\leq\;
	\RCData(\mathcal{F}) + \RCData(\mathcal{G}).
	\]
\end{lemma}

\begin{proof}
	For any realization of the Rademacher random variables $\eps_1,\ldots,\eps_n$,
	\[
	\sup_{h \in \mathcal{F} \cup \mathcal{G}} \frac{1}{n} \sum_{i=1}^n \eps_i h(Z_i)
	= \max\!\left\{
	\sup_{f \in \mathcal{F}} \frac{1}{n} \sum_{i=1}^n \eps_i f(Z_i),
	\;
	\sup_{g \in \mathcal{G}} \frac{1}{n} \sum_{i=1}^n \eps_i g(Z_i)
	\right\}.
	\]
	The conclusion follows from point (c) in Exercise~3.8 in \citet{mohri2018foundations}.
\end{proof}

\begin{lemma}[Rademacher Complexity of a linear combination]
	\label{lem:RC-lin}
	Let $\cG$ and be a class of functions $g : \cZ \to \R$ for some space $\spaceZ$. For any i.i.d. sample $Z_1, \dots, Z_n$ on $\spaceZ$, we have
	\[
	\RCData \lp \lin_T^c(\cG) \rp
	\;\leq\;
	(1+Tc) \, \RCData(\cG).
	\]
\end{lemma}

\begin{proof}
	We have
	\begin{align*}
		\E \lc \sup_{g \in \lin_T^c(\cG)} \frac{1}{n} \sum_{i=1}^n \varepsilon_i g(Z_i) \rc
		&= \E \lc \sup_{g_0, g_1, \ldots, g_T \in \cG} \sup_{\bm{\alpha} \in \R^T_+: \|\bm{\alpha}\|_1 \leq Tc} \frac{1}{n} \sum_{i=1}^n \varepsilon_i  \lp g_0 + \sum_{t=1}^T \alpha_t g_t(Z_i) \rp \rc \\
		&\leq (1+Tc) \; \E \lc \sup_{g_0,g_1, \ldots, g_T \in \cG} \sup_{\bm{\alpha} \in \R^{T+1}_+: \|\bm{\alpha}\|_1 = 1} \sum_{i=1}^n \varepsilon_i \sum_{t=0}^T \alpha_t g_t(Z_i) \rc \\
		&= (1+Tc) \; \RCData \lp \conv(\cG) \rp = (1+Tc) \; \RCData(\cG),
	\end{align*}
	where $\conv(\cG)$ denotes the convex hull of the class $\cG$ and where the last equality follows from Lemma 7.4 in \citet{mohri2018foundations}.
\end{proof}

\section{Proofs}
\label{app:proofs}

\begin{proof}[Proof of Theorem \ref{thm:stochastic}]
	We start by showing that under Assumption \ref{ass:fisher}, for any $\eps > 0$,
	\begin{equation}
		\label{eq:L1-bound}
		\P \lc \int_{\suppX} \| \hF_n(x) - \F^u(x) \|_\infty \diff P_X(x) > \eps \rc
		\leq
		\P \lc \sup_{\F \in \outputspace} D_n(\F) > \ist \eps \rc,  
	\end{equation}
	where for $\F \in \outputspace$,
	\[
		D_n(\F) = \left\| \nabla_{\bphi} \hR_n^{\hu }(\F) -\nabla_{\bphi} R^{u}(\F)\right\|_\infty.
	\]
	To this end, introduce the functional
	\[
		s > 0 \mapsto \mathcal{J}_G(s) = \nabla_{\bphi} R^u(\F^u + s \G)
	\]
	associated with $\G \in \lin(\cH)$. By definition, we now $\mathcal{J}_G(0) = (0,0)^\tr$. By Taylor's theorem around $s=0$, we have for $s>0$,
	\[
		\mathcal{J}_G(s) = s \frac{n}{k_n} \E \lc  \nabla_{\bphi}^2 L \lp \widetilde{\F}(\X), Y - u(\X ; \tfrac{k_n}{n}) \rp^\tr \G(\X) \I \lacc  Y > u(\X; \tfrac{k_n}{n}) \racc \rc,
	\]
	where $\widetilde{\F}(\X) \in \reparamspace$.
	For $s=1$ and $\G = \hF_n - \F^u \in \lin(\cH)$, we get by Assumption \ref{ass:fisher},
	\[
		\nabla_{\bphi} R^u(\hF_n) \geq \mathscr{I} \int_{\suppX} \| \hF_n(\x) - \F^u(\x) \|_\infty \, \diff P_{\X}(\x),
	\]
	from which \eqref{eq:L1-bound} follows.
	
	We split the analysis of the process $\{D_n(\F): \F \in \outputspace\}$ in two parts. By the triangle inequality,
	\begin{equation}
		\label{eq:split_risk}
		\sup_{\F \in \outputspace} D_n(\F)
		\leq  
		\sup_{\F \in \outputspace} D_{n,1}(\F)
		+ \sup_{\F \in \outputspace} D_{n,2}(\F),
	\end{equation}
	where for $\F \in \outputspace$,
	\[
		D_{n,1}(\F) = \left\| \nabla_{\bphi} \hR_n^{\hu}(\F) - \nabla_{\bphi} \hR_n^{u}(\F) \right\|_\infty 
	\]
	and
	\[
		D_{n,2}(\F) = \left\| \nabla_{\bphi} \hR_n^{u}(\F) -\nabla_{\bphi} R^{u}(\F) \right\|_\infty. 
	\]
	The first process is concerned with the stochastic error originating from the estimation of the intermediate quantile $u(\cdot ; \tfrac{k_n}{n})$ while the second process corresponds to the error coming from empirical estimation of the pre-asymptotic risk $R^u$.
	We deal with the two terms on the right-hand side of \eqref{eq:split_risk} below.
	
	\paragraph*{Analysis of the process $\boldsymbol{D_{n,1}}$.}
	
	Fix $M = M(n) > 0$. We consider separately the quantities
	\begin{multline}
		\label{eq:first-lower-term}
		 \underline{D_{n,1}} (\F)
		= 
		\Bigg\| \frac{1}{k_n}
		\sum_{i=1}^n \nabla_{\bphi} \Big[  L \lp \F(\X_i), Y_i - \hu(\X_i;\tfrac{k_n}{n}) \rp \I\{Y_i > \hu(\X_i;\tfrac{k_n}{n})\} \\
		-  L \lp \F(\X_i), Y_i - u(\X_i;\tfrac{k_n}{n}) \rp \I\{Y_i > u(\X_i;\tfrac{k_n}{n}\} 
		\Big] \I \lacc \Phi(Y_i) \leq M \racc
		\Bigg\|_\infty
	\end{multline}
	and
	\begin{multline}
		\label{eq:first-upper-term}
		\overline{D_{n,1}} (\F)
		= 
		\Bigg\| \frac{1}{k_n}
		\sum_{i=1}^n \nabla_{\bphi} \Big[  L \lp \F(\X_i), Y_i - \hu(\X_i;\tfrac{k_n}{n}) \rp \I\{Y_i > \hu(\X_i;\tfrac{k_n}{n})\} \\
		-  L \lp \F(\X_i), Y_i - u(\X_i;\tfrac{k_n}{n}) \rp \I\{Y_i > u(\X_i;\tfrac{k_n}{n}\} 
		\Big] \I \lacc \Phi(Y_i) > M \racc
		\Bigg\|_\infty
	\end{multline}
	
	To deal with the sup associated with \eqref{eq:first-lower-term}, we introduce the sets
	\begin{equation}
	\label{eq:tail-set}
		E_n^u
		=
		\lacc
			(\x,y) \in \suppX \times \R_+: \; y > u(\x;\tfrac{k_n}{n})
		\racc
	\end{equation}
	and
	\[
		E_n^{\hu}
		=
		\lacc
		(\x,y) \in \suppX \times \R_+: \; y > \hu(\x;\tfrac{k_n}{n})
		\racc.
	\]
	The triangle inequality implies that for $\F \in \outputspace$,
	\begin{align}
		\lefteqn{\underline{D_{n,1}} (\F)} \nonumber \\
		&\leq \frac{1}{k_n} \sum_{i=1}^{n}  \left\| \nabla_{\bphi}L \lp \F(\X_i), Y_i - u(\X_i;\tfrac{k_n}{n}) \rp \right\|_\infty \I \lacc (\X_i,Y_i) \in E_n^{\hu} \triangle  E_n^u \racc  \I \lacc \Phi(Y_i) \leq M \racc \label{eq:lower-symdif-term} \\
		&\ + \frac{1}{k_n} \sum_{i=1}^{n}  \left\| \nabla_{\bphi} \lc L \lp \F(\X_i), Y_i - \hu(\X_i;\tfrac{k_n}{n}) \rp - L \lp \F(\X_i), Y_i - u(\X_i;\tfrac{k_n}{n}) \rp \rc \right\|_\infty \cdot \nonumber \\ 
        &\ \hskip0.7\textwidth \I\{(\X_i,Y_i) \in E_n^u\} \label{eq:lower-diff-u-term}
	\end{align}
	We bound both terms below with high probability.
	
	Consider \eqref{eq:lower-symdif-term}.
	Fix $\delta \in (0,1)$. By Assumption \ref{ass:quantile-estimation}, there exists an event $\event_1 \in \cA$ such that $\P(\event_1) \geq 1-\delta$ on which for any $(\x,y) \in \suppX \times \R_+$,
	\[
		\I \lp (\x,y) \in E_n^{\hu} \triangle  E_n^u \rp
		\leq \I \lp (\x,y) \in E_n^{u,\delta} \rp,
	\]
	where
	\[
		E_n^{u,\delta}
		=
		\lacc
		(\x,y) \in \suppX \times \R_+:  u(\x;\tfrac{k_n}{n}) \{1-a_n(\delta)\} < y < u(\x;\tfrac{k_n}{n}) \{1+a_n(\delta)\}
		\racc.
	\]
	Since $\sum_{i=1}^{n} \I \lacc (\X_i,Y_i) \in E_n^{u,\delta} \racc \sim \Bin(n,p_{n,\delta})$ where $p_{n,\delta} = P_{(\X,Y)}(E_n^{u,\delta})$. By Lemma \ref{lem:tail-bound-bin}, there exists $\event_2 \in \cA$ with $\P(\event_2) \geq 1-\delta$ on which
	\[
		\sum_{i=1}^{n} \I \lacc (\X_i,Y_i) \in E_n^{u,\delta} \racc \leq \ee np_{n,\delta} + \sqrt{\frac{6\ee}{5} np_{n,\delta} \log(1/\delta)}.
	\]
	By Lemma \ref{lem:upper-fun}, on $\event_1 \cap \event_2$,
	\[
		\eqref{eq:lower-symdif-term} \; \leq \;  \frac{M}{k_n} \lp \ee n p_{n,\delta}  + \sqrt{\frac{6\ee}{5} np_{n,\delta} \log(1/\delta)} \rp.
	\]
	It remains to compute 
    \[
    	p_{n,\delta}
		= \int_{\suppX} 
        \Big[
            \bF_{Y \mid \X = \x} \lp u(\x;\tfrac{k_n}{n})(1-a_n(\delta)) \rp
                -
            \bF_{Y \mid \X = \x} \lp u(\x;\tfrac{k_n}{n})(1 + a_n(\delta)) \rp
        \Big] 
        \diff P_{\X}(\x) 
    \]
    For $s > 0$, we will use the notation
    \[
        g_{n,\x}(s) = \bF_{Y \mid \X = \x}  \lp u(\x;\tfrac{k_n}{n}) \cdot s \rp.
    \]
    Obviously, $g(1) = k_n/n$. Furthermore, it follows from Assumptions \ref{ass:tail-of-Y} and \ref{ass:sv-function} that for $s > 0$,
    \begin{align*}
        g_{n,\x}'(s)
        &= - f_{Y|\X=\x} \lp u(\x;\tfrac{k_n}{n}) \cdot s \rp \cdot u(\x;\tfrac{k_n}{n}) \\
        &= - s^{-1} g_{n,\x}(s) \lp \frac{1}{\xi_0(\x)} - r(u(\x;\tfrac{k_n}{n}) \cdot s \mid \x) \rp.
    \end{align*}
    It follows from Taylor's theorem that for some $s_n(\delta) \in [1-a_n(\delta), 1+a_n(\delta)]$,
    \begin{align*}
        \lefteqn{g_{n,\x}(1-a_n(\delta)) - g_{n,\x}(1+a_n(\delta))} \\
        &= - 2 a_n(\delta) g_{n,\x}' \lp s_n(\delta) \rp \\
        &= \frac{2}{s_n(\delta)} a_n(\delta) g_{n,\x} \lp s_n(\delta) \rp \lp \frac{1}{\xi_0(\x)} - r(u(\x;\tfrac{k_n}{n}) \cdot s_n(\delta) \mid \x) \rp \\
        &\leq 4 a_n(\delta) g_{n,\x} \lp s_n(\delta) \rp \lp \frac{1}{\xi_L} + \labs r(u(\x;\tfrac{k_n}{n}) / 2 \mid \x) \rabs \rp \\
        &\leq 8 a_n(\delta) \frac{k_n}{n} \lp 1 + \xi_L^{-1} \rp
    \end{align*}
    where the last inequality is based on the assumptions in the statement of the theorem. Integrating the bound with respect to $P_{\X}$ leads to
    \[
        p_{n,\delta} \leq 8 \lp 1 + \xi_L^{-1} \rp a_n(\delta) \frac{k_n}{n},
    \]
    so that for any $\F \in \outputspace$,
    \[
        \eqref{eq:lower-symdif-term} \; \leq \;  M \lp 8 \ee  \lp 1 + \xi_L^{-1} \rp a_n(\delta)  + \sqrt{\frac{6\ee}{5 k_n} 8  \lp 1 + \xi_L^{-1} \rp a_n(\delta) \log(1/\delta)} \rp.
    \]
	
	Now consider \eqref{eq:lower-diff-u-term}. An application of Lemma \ref{lem:lip-z} leads to the following bound on $\event_1$,
	\[
		\eqref{eq:lower-diff-u-term} 
        \; \leq \; 
        \kappa  a_n(\delta) \frac{1}{k_n} \sum_{i=1}^n \I\{(\X_i,Y_i) \in E_n^u\}.
	\]
	As $\sum_{i=1}^n \I\{(\X_i,Y_i) \in E_n^u\}  \sim \Bin(n,k_n/n)$, it follows from Lemma \ref{lem:tail-bound-bin} that there exists an event $\event_3 \in \cA$ with $\P(\event_3) \geq 1-\delta$ on which
	\[
		\sum_{i=1}^n \I\{(\X_i,Y_i) \in E_n^u\} \leq \ee k_n + \sqrt{\frac{6\ee}{5} k_n \log(1/\delta)}.
	\]
	We get that on $\event_1 \cap \event_3$,
	\[
		\eqref{eq:lower-diff-u-term} 
		\; \leq \; 
		\kappa  a_n(\delta) 
		\lacc \ee + \sqrt{\frac{6\ee}{5k_n} \log(1/\delta)} \racc.
	\]
	Concluding the analysis of \eqref{eq:first-lower-term}, we get that on $\cap_{j=1}^3 \event_j$,
	\begin{align}
		\lefteqn{\sup_{\F \in \outputspace} \underline{D_{n,1}} (\F)} \nonumber \\
		&\leq 
		a_n(\delta) 
		\Bigg\{
		M 8 \ee  \lp 1 + \xi_L^{-1} \rp + \kappa(\reparamspace) \lc \ee + \sqrt{\frac{6\ee}{5k_n} \log(1/\delta)} \rc
		\Bigg\} \nonumber \\
		&\qquad + M \sqrt{\frac{6\ee}{5 k_n} 8  \lp 1 + \xi_L^{-1} \rp a_n(\delta) \log(1/\delta)}	\label{eq:first-lower-final}
	\end{align}
	
	We now turn to the sup associated with \eqref{eq:first-upper-term}. Similarly to what has been done previously, we may rely on Lemmas \ref{lem:upper-fun} and \ref{lem:lip-z} to show that on $\event_1$,
	\begin{align}
		\sup_{\F \in \outputspace} \overline{D_{n,1}} (\F)
		&\leq \frac{1}{k_n} \sum_{i=1}^{n}  \Phi(Y_i) \I \lacc  \Phi(Y_i) > M \racc \label{eq:upper-symdif-term} \\
		&\ + \kappa  a_n(\delta) \frac{1}{k_n} \sum_{i=1}^{n} \I \lacc \Phi(Y_i) > M \racc \label{eq:upper-diff-u-term}.
	\end{align}
	
	For \eqref{eq:upper-symdif-term}, we rely on Lemma \ref{lem:upper-fun}.
	Let $m_\rho = M_{\Phi(Y)}(\rho)$. Applying Markov's inequality, we have that for any $\eps > 0$,
	\begin{equation*}
		\P \lc \sum_{i=1}^{n}  \Phi(Y_i) \I  \lacc  \Phi(Y_i) > M \racc > n\eps \rc \leq \frac{\lp \E \lc \exp  \lacc (\rho/2) \Phi(Y) \I \lacc  \Phi(Y_i) > M \racc \racc \rc \rp^n}{\exp(n\eps\rho/2)}.
	\end{equation*}
	By Cauchy--Schwarz inequality
	\begin{align*}
		\lefteqn{\E \lc \exp  \lacc (\rho/2) \Phi(Y) \I \lacc  \Phi(Y_i) > M \racc \racc \rc} \\
		&= 1 + \E \lc \exp  \lacc (\rho/2) \Phi(Y) \racc \I \lacc  \Phi(Y_i) > M \racc \rc \\
		&\leq 1 + \sqrt{m_\rho \cdot \P \lc \Phi(Y_i) > M  \rc } \\
		&\leq 1 + \frac{m_\rho}{n}
	\end{align*}
	with the choice
	\begin{equation}
	\label{eq:M-choice}
		M = M(n) = \frac{2}{\rho} \log(n) = \frac{2}{\rho \alpha} \log(k_n).
	\end{equation}
	Since for any $u>0$, we have $1+u \leq \exp(u)$, we get for $\eps > 0$,
	\[
		\P \lc \sum_{i=1}^{n}  \Phi(Y_i) \I  \lacc  \Phi(Y_i) > M \racc > n\eps \rc
		 \leq
		 \exp \lacc m_\rho - n\eps\rho/2 \racc.
	\]
	This means that there exists $\event_4 \in \cA$ with $\P(\event_4) \geq 1-\delta$ on which
	\[
		\eqref{eq:upper-symdif-term} \; \leq \;
		\frac{2}{\rho k_n} \lp \log(1/\delta) + m_\rho \rp.
	\]
	
	For \eqref{eq:upper-diff-u-term}, since with our choice of $M$, we have $\sum_{i=1}^{n} \I \lacc \Phi(Y_i) > M \racc \sim \Bin(n, p_M)$ where $p_M \leq m_\rho/n^2$. We may apply Lemma \ref{lem:tail-bound-bin} to obtain that there exists an event $\event_5 \in \cA$ with $\P(\event_5) \geq 1-\delta$ on which
	\[
		\sum_{i=1}^{n} \I \lacc \Phi(Y_i) > M \racc \leq \frac{\ee m_\rho}{n}  + \sqrt{\frac{6\ee m_\rho}{5n} \log(1/\delta)}
	\]
	and thus on $\event_5$,
	\[
		\eqref{eq:upper-diff-u-term} \; \leq \;
		\kappa  \frac{a_n(\delta)}{k_n} \lp \frac{\ee m_\rho}{n}  + \sqrt{\frac{6\ee m_\rho}{5n} \log(1/\delta)} \rp
	\]
	Concluding the analysis of \eqref{eq:first-upper-term}, we get that on $\event_1 \cap \event_4 \cap \event_5$,
	\begin{multline}
	\label{eq:first-upper-final}
		\sup_{\F \in \outputspace} \overline{D_{n,1}} (\F)
		\leq
		\frac{1}{k_n}
		\Bigg\{
			\frac{2}{\rho} \lp \log(1/\delta) + m_\rho \rp  \\
			+
			\kappa   a_n(\delta) \lp \frac{\ee m_\rho}{n}  + \sqrt{\frac{6\ee m_\rho}{5n} \log(1/\delta)} \rp
		\Bigg\}.
	\end{multline}
	
	\paragraph*{Analysis of the process $\boldsymbol{D_{n,2}}$.}
	We introduce the class of functions
	\begin{multline*}
		\cG
		= 
		\Bigg\{
		g_{j,\F}: (\x,y) \in \suppX \times \R_+ \mapsto
		g_{j,\F}(\x,y) \\ = \Dot{L}_j (\F(\x),y-u(\x;k_n/n)) \I(y > u(\x;k_n/n)),
		\F \in \outputspace, j \in [2]
		\Bigg\}
	\end{multline*}
	Let $\widehat{P}_{n, (\X,Y)}$ be the empirical measure on the sample $\data$ and recall $P_{(\X,Y)}$ the true underlying measure of $(\X,Y)$ on $\suppX \times \R_+$. It is easy to see that
	\begin{equation}
		\label{eq:Dn2-supG}
		\sup_{F \in \outputspace} D_{n,2}(\F)
		\leq \;
		\frac{n}{k_n}
		\sup_{g \in \cG} 
		\left|
		\Pdata(g) - \Ptrue(g)
		\right|.
	\end{equation}
	Here also we split the analysis in two parts depending on whether the random variable $\Phi(Y)$ is small or not, i.e., we will consider separately the random quantities
	\begin{equation}
	\label{eq:sup-G-lower}
		\underline{Z} = \sup_{g \in \cG} 
		\left|
		\frac{1}{n} \sum_{i=1}^n g(\X_i,Y_i) \I \lacc \Phi(Y_i) \leq M \racc
		-
		\int_{\suppX \times \R_+} g(\x,y) \I \lp \Phi(y) \leq M \rp \diff \Ptrue(\x,y)
		\right|
	\end{equation}
	and
	\begin{equation}
	\label{eq:sup-G-upper}
        \overline{Z} =
		\sup_{g \in \cG} 
		\left|
		\frac{1}{n} \sum_{i=1}^n g(\X_i,Y_i) \I \lacc \Phi(Y_i) > M \racc
		-
		\int_{\suppX \times \R_+} g(\x,y) \I \lp \Phi(y) > M \rp \diff \Ptrue(\x,y)
		\right|
	\end{equation}
	where $M = M(n) = (2/\rho\alpha) \log(k_n)$ as in \eqref{eq:M-choice}.
	
	For \eqref{eq:sup-G-lower}, we start by applying Theorem \ref{thm:concentration-expectation} with $c \leq 2M/n$ and $\sigma^2 \leq 4M^2 k_n/n^2$ which ensures that there exists an event $\event_6$ with $\P(\event_6) \geq 1-\delta$ on which
	\[
		\underline{Z} 
        \; \leq \;
        \E \lc \underline{Z} \rc +
        \frac{4 \log k_n}{\rho\alpha n} \lp \sqrt{k_n\log(1/\delta)} + \frac{2}{3} \log(1/\delta) \rp.
	\]
	By a slight abuse of notation, if $E \subset \suppX \times \R_+$, let us write
	\[
		\RCData(\cG \I_E) = \E \lc \sup_{g \in \cG} \frac{1}{n} \sum_{i=1}^n \eps_i g(\X_i,Y_i) \I \lacc (\X_i,Y_i) \in E \racc \rc.
	\]
	Also, for a class  $\cH$ of vector valued functions $\h = (h_1,h_2) : \spaceZ \mapsto \R_+^2$, define the Rademacher complexity as
	\[
		\RCData(\cH)
		= \E \lc \max_{j = 1,2} \sup_{\h \in \cH} \frac{1}{n} \sum_{i=1}^n \varepsilon_i h_j(Z_i) \rc,
	\]
    By Lemma 2.3.1 in \citet{vdVW1996weak} which ensures that
    \[
        \E[\underline{Z}] \leq 2 \RCData(\cG^M \I_{E_n^u}),
    \]
	where we recall $E_n^u$ in \eqref{eq:tail-set} which satisfies $\Ptrue(E_n^u) \leq k_n/n$ and $\cG^M = \cG^M_1 \cup \cG^M_2$ with
	\begin{multline*}
		\cG^M_j
		= 
		\Bigg\{
			g_{j,\F}: (\x,y) \in \suppX \times \R_+ \mapsto
			g_{j,\F}(\x,y) \\ = \Dot{L}_j (\F(\x),y-u(\x;k_n/n)) \I(\Phi(y) < M),
			\F \in \outputspace
		\Bigg\}
	\end{multline*}
	for $j \in [2]$. The conditioning trick in Lemma \ref{lem:CT} implies that if $K = \sum_{i=1}^n \I \lacc (\X_i,Y_i) \in E_n^u \racc \sim \Bin(n,k_n/n)$ denotes the number of points in $E_n^u$, we have
	\begin{align*}
		\RCData(\cG^M \I_{E_n^u})
		&= \E \lc \sup_{g \in \cG^M} \frac{1}{n} \sum_{i=1}^n \eps_i g(\X_i,Y_i) \I \lacc (\X_i,Y_i) \in E_n^u \racc \rc \\
		&= \E_K \lc \E \lc \sup_{g \in \cG^M} \frac{1}{n} \sum_{i=1}^n \eps_i g(\X_i,Y_i) \I \lacc (\X_i,Y_i) \in E_n^u \racc \Big| K \rc \rc \\
		&= \E \lc \frac{K}{n}  \sup_{g \in \cG^M} \frac{1}{K} \sum_{i=1}^n \eps_i g (\widetilde{\X}_i,\widetilde{Y}_i) \rc \\
		&= \sum_{k=0}^{n} \frac{k}{n} \widetilde{\mathfrak{R}}_k(\cG^M)  p_K(k),
	\end{align*}
	where the random pair $(\widetilde{\X},\widetilde{Y})$ is distributed as $P_{(\X,Y), E_n^u}$ (using the same notation as in the statement of Lemma \ref{lem:CT}) and where $p_K(k) = \P(K=k)$ for $k \in [0:n]$. By Lemma \ref{lem:RC-union}, we have
	\[
		\widetilde{\mathfrak{R}}_k(\cG^M)
		\leq
		\widetilde{\mathfrak{R}}_k(\cG_1^M) + \widetilde{\mathfrak{R}}_k(\cG_2^M)
	\]
	The functions $g_j \in \cG^M_j$ are in one-to-one correspondence with functions $\F \in \outputspace$ and, by Lemma \ref{lem:lip-param}, we have that the maps
	\[
		\bphi \in \reparamspace \mapsto \Dot{L}_j(\bphi, y-u) \I \lp \Phi(y) \leq M \rp
	\]
	are $\mu_j$-Lipschitz with $\mu_2$ depending on $n$. We get from Talagrand contraction principle \citep[Lemma 5.7]{mohri2018foundations} that
	\[
		\widetilde{\mathfrak{R}}_k(\cG_j^M) \leq \mu_j \, \widetilde{\mathfrak{R}}_k(\outputspace).
	\]
	As the action of clamping to $\reparamspace$ is a $1$-Lipschitz operation, we have
	\[
		\widetilde{\mathfrak{R}}_k(\outputspace) = \widetilde{\mathfrak{R}}_k(\lin_T^c(\cH))
	\]
	It directly follows from Lemma \ref{lem:RC-lin} that
	\[
		\widetilde{\mathfrak{R}}_k(\lin_T^c(\cH)) \leq (1+Tc) \widetilde{\mathfrak{R}}_k(\cH).
	\]
	Since each function $\h \in \cH$ is of the form $\h = (h^{(1)}, h^{(2)})$ for some $h^{(j)} \in \cH_j$, we get from point (c) in Exercise 3.8 of \citet{mohri2018foundations} that
	\[
		 \widetilde{\mathfrak{R}}_k(\cH) \leq  \widetilde{\mathfrak{R}}_k(\cH_1) +  \widetilde{\mathfrak{R}}_k(\cH_2).
	\]
	Under Assumption \ref{ass:hypothesis}, we have for $j \in [2]$, by Proposition 1 in \cite{einmahl2005uniform},
	\[
		 \widetilde{\mathfrak{R}}_k(\cH_j)
		 \leq
		 \dst \sqrt{\frac{2 \VC_j \|H_j\|_{L_2(\suppX,P_{\X})} \log \lp \kst \VC_j (16\ee)^{\VC_j} \rp}{k}},
	\]
	for some absolute constants $\dst > 0$.
	Combining the above results we get the following on the Rademacher complexity 
	\begin{align*}
		\lefteqn{\RCData(\cG^M \I_{E_n^u})
		= \sum_{k=0}^{n} \frac{k}{n} \widetilde{\mathfrak{R}}_k(\cG^M)  p_K(k)} \\
		&\leq  \frac{(\mu_1 + \mu_2) (1 + Tc) \dst}{n}  \sum_{k=0}^{n} 
		\Bigg( 
			\sqrt{2 k \VC_1 \|H_1\|_{L_2(\suppX,P_{\X})} \log \lp \kst \VC_1 (16\ee)^{\VC_1} \rp} \\
			&\qquad + \sqrt{2 k \VC_2 \|H_2\|_{L_2(\suppX,P_{\X})} \log \lp \kst \VC_2 (16\ee)^{\VC_2} \rp}
		\Bigg) p_K(k) \\
		&\leq \frac{(\mu_1 + \mu_2) (1 + Tc) \dst}{n} \sqrt{2k_n}
		\Bigg( 
			\sqrt{\VC_1 \|H_1\|_{L_2(\suppX,P_{\X})} \log \lp \kst \VC_1 (16\ee)^{\VC_1} \rp} \\
			&\qquad + \sqrt{\VC_2 \|H_2\|_{L_2(\suppX,P_{\X})} \log \lp \kst \VC_2 (16\ee)^{\VC_2} \rp}
		\Bigg),
	\end{align*}
	where the last inequality follows from Jensen's inequality. We finally get that on $\event_6$,
	\begin{multline*}
		\underline{Z} \; \leq \;
		2\frac{\sqrt{k_n}}{n}
		\Bigg( 
		(\mu_1 + \mu_2) (1 + Tc) \sqrt{2} \dst 
		\Big( 
		\sqrt{\VC_1 \|H_1\|_{L_2(\suppX,P_{\X})} \log \lp \kst \VC_1 (16\ee)^{\VC_1} \rp} \\
		+ \sqrt{\VC_2 \|H_2\|_{L_2(\suppX,P_{\X})} \log \lp \kst \VC_2 (16\ee)^{\VC_2} \rp}
		\Big)
		+ \frac{2\log k_n}{\rho\alpha} \sqrt{\log(1/\delta)}
		\Bigg) + \frac{8\log k_n}{3\rho\alpha n} \log(1/\delta)
	\end{multline*}
	
	We now deal with $\overline{Z}$ in \eqref{eq:sup-G-upper}. Trivially, we get the bound
	\[
		\overline{Z}
		\; \leq \;
		\frac{1}{n} \sum_{i=1}^n \Phi(Y_i) \I \lacc \Phi(Y_i) > M \racc + \E \lc \Phi(Y) \I \lacc \Phi(Y_i) > M \racc \I \lacc (\X,Y) \in E_n^u \racc \rc.  
	\]
	Recall that on $\event_4$, we have
	\[
		\frac{1}{n} \sum_{i=1}^n \Phi(Y_i) \I \lacc \Phi(Y_i) > M \racc
		\leq
		\frac{2}{\rho n} \lp \log(1/\delta) + m_\rho \rp.
	\]
	Furthermore, by Cauchy--Scwharz inequality and the fact that $\Ptrue(E_n^u) = k_n /n$, $\P(\Phi(Y) > M) \leq m_\rho/n^2$ and for any $j \in \N_0$,
	\[
		\E \lc \Phi(Y)^j \rc \leq \frac{j! m_\rho}{\rho^j},
	\]
	we can bound the expectation
	\begin{align*}
		 \E \lc \Phi(Y) \I \lacc \Phi(Y_i) > M \racc \I \lacc (\X,Y) \in E_n^u \racc \rc
		 &\leq \sqrt{\frac{k_n}{n} \E \lc  \Phi(Y)^2 \I \lacc \Phi(Y_i) > M \racc \rc} \\
		 &\leq \frac{\sqrt{k_n}}{n} m_\rho^{1/4} \E \lc \Phi(Y)^4 \rc^{1/4} \\
		 &\leq \frac{3 \sqrt{m_\rho}}{\rho} \frac{\sqrt{k_n}}{n},
	\end{align*}
	where in the last inequality we use the bound $(24)^{1/4} \leq 3$.
	Summarizing, we get the following inequality on $\event_4$,
	\begin{align*}
		\overline{Z}
		\; &\leq \;
		\frac{2}{\rho n} \lp \log(1/\delta) + m_\rho \rp + \frac{3 \sqrt{m_\rho}}{\rho} \frac{\sqrt{k_n}}{n} \\
		&\leq \frac{3}{\rho n} \lp \sqrt{k_n m_\rho} +  \log(1/\delta) + m_\rho \rp.
	\end{align*}
	
	Combining, the bounds on \eqref{eq:sup-G-lower} and \eqref{eq:sup-G-upper} with \eqref{eq:Dn2-supG} leads to our final bound on the process $D_{n,2}$ which holds on the event $\event_4 \cap \event_6$.
	
	\paragraph*{Conclusion of the proof.}  Combining the bounds on $D_{n,1}$ and $D_{n,2}$, we get that on $\cap_{j=1}^6 \event_j$ which satisfies $\P \lp \cap_{j=1}^6 \event_j \rp \geq 1 - 6\delta$,
	\[
		\int_{\suppX} \| \hF_n(x) - \F^u(x) \|_\infty \diff P_X(x)
		\lesssim
		\mathscr{I} \lacc B_{n,1}(\delta) + B_{n,1}(\delta) \racc,
	\]
	where
	\[
		B_{n,1}(\delta)
		= 
        \log k_n
        \lp
        \sqrt{ \frac{a_n(\delta)}{k_n} \log(1/\delta)}
            +
        a_n(\delta)
        \rp,
	\]
	and
	\[
		B_{n,2}(\delta)
		= \frac{\log k_n}{\sqrt{k_n}} \lc (1+Tc)
        + \sqrt {\log(1/\delta)} \rc.
	\]
	Taking $\delta/6$ instead of $\delta$ concludes the proof.
\end{proof}

\begin{proof}[Proof of Theorem \ref{thm:bias}]
	Using a similar argument to the one in the beginning of the proof of Theorem \ref{thm:stochastic}, we may show that under Assumption \ref{ass:fisher} for any $\eps > 0$,
	\[
		\P \lc 	\int_{\suppX} \| \hF_n(x) - \F^u(x) \|_\infty \diff P_X(x) > \eps \rc
		\leq 
		\P \lc \| \nabla_{\bphi} R^{u}(\F_*) \|_\infty > \mathscr{I} \eps \rc. 
	\]
	As such, we analyze the behavior of
	\[
		\left\| \nabla_{\bphi} \lp R^u(\F) - R(\F) \rp \right\|
	\]
	for fixed $\F \in \outputspace$. To achieve this, we need to bound
	\[
		\labs 
			\frac{n}{k_n} \E_{\X Y} \lc \Dot{L}_j \lp \F(\X), Y - u(\X;\tfrac{k_n}{n}) \rp \I\{Y > u(\X;\tfrac{k_n}{n}) \} \rc 
			-
			\E_{\X Z} \lc \Dot{L}_j \lp \F(\X), Z \rp \rc
		\rabs
	\]
	for $j \in [2]$, where $(\X,Y) \sim \Ptrue$ and where $Z \mid \X = \x \sim \operatorname{GP}(\btheta_0(\x))$. The latter quantity is bounded by
    \begin{equation}
        \label{eq:bias-process}
        \int_{\suppX}
		\labs
			\E_{Y|\X=\x} \lc \Dot{L}_j \lp \F(\x), Y - u(\x ; \tfrac{k_n}{n}) \rp \mid Y > u(\x;\tfrac{k_n}{n}) \rc
			-
			\E_{Z|\X=\x} \lc \Dot{L}_j \lp \F(\x), Z \rp \rc
		\rabs
		\diff P_{\X}(\x).
    \end{equation}
	Now to continue further, we consider separately the cases $j=1$ and $j=2$.
	
	\paragraph*{Analysis of the term $\boldsymbol{j=1}$.}
	Recall that for $\bphi = (\nu,\xi) \in \reparamspace$ and $z > 0$,
	\[
        \Dot{L}_1(\bphi,z)
        = \frac{1}{\nu} \lacc 1 - \frac{(\xi+1)^2 z}{\nu + \xi(\xi+1)z}\racc
    \]
	Hence if $\F(\x) = (\nu(\x),\xi(\x))$, we get the following bound on \eqref{eq:bias-process} for $j=1$,
	\begin{multline*}
		\int_{\suppX}
		\lp \frac{\lp \xi(\x) + 1 \rp^2}{\nu(\x)} \rp
		\Bigg|
		\E_{Y|\X=\x} \lc \frac{Y - u(\x;\tfrac{k_n}{n})}{\nu(\x)+\xi(\x)(\xi(\x) + 1)(Y - u(\x;\tfrac{k_n}{n})} \mid Y > u(\x;\tfrac{k_n}{n}) \rc \\
		-
		\E_{Z|\X=\x} \lc \frac{Z}{\nu(\x)+\xi(\x)(\xi(\x) + 1) Z} \rc
		\Bigg|
		\diff P_{\X}(\x).
	\end{multline*}
	Using the formula
	\begin{equation}
		\label{eq:formula-E}
			E[Y] = \int_0^\infty \overline{F}_Y(t) \; \diff t
	\end{equation}
	which holds for any positive valued real random variables $Y$ and the fact that
	\[
		\frac{z}{a+bz} > t \iff z > \frac{ta}{1-bt}
	\]
	for $a,b,t,z>0$ such that $t<1/b$ and never otherwise, we get that the previous expression equals
	\begin{multline*}
		\int_{\suppX}
		\lp \frac{ \lp \xi(\x) + 1 \rp^2}{\nu(\x)} \rp
		\int_{t=0}^{\frac{1}{\xi(\x)(\xi(\x) + 1)}}
		\Bigg|
		\frac{\overline{F}_{Y|\X=\x} \lp u(\x;\tfrac{k_n}{n}) + \frac{t\nu(\x)}{1-t\xi(\x)(\xi(\x) + 1} \rp}{\overline{F}_{Y|\X=\x} \lp u(\x;\tfrac{k_n}{n}) \rp}
		- \\
		\overline{F}_{Z|\X=\x} \lp \frac{t\nu(\x)}{1-t\xi(\x)(\xi(\x)+1)} \rp
		\Bigg|
		\; \diff t \,
		\diff P_{\X}(\x).
	\end{multline*}
	
	Let us introduce the following notation for $(\x,t) \in \suppX \times \R_+$ such that $t < \tfrac{1}{\xi(\x)(\xi(\x) + 1)}$
	\[
		s(t,\x) = \frac{t\nu(\x)}{1-t\xi(\x)(\xi(\x)+1)} > 0.
	\]
	Then, by Theorem \ref{thm:P-B-dH}, since we may take
	\[
		\sigma_0(u(\x),\x)
		=
		u(\x) \xi_0(\x),
	\]
    we get that our quantity of interest is bounded by
	\begin{equation*}
		\int_{\suppX}
		\lp \frac{ \lp \xi(\x) + 1 \rp^2}{\nu(\x)} \rp
		\int_{t=0}^{\frac{1}{\xi(\x)(\xi(\x) + 1)}}
		\lp \frac{u(\x;\tfrac{k_n}{n})}{s(t,\x) +u(\x;\tfrac{k_n}{n})} \rp^{1/\xi_0(\x)}
		\Bigg|
			\frac{\ell(u(\x;\tfrac{k_n}{n}) + s(t,\x) \mid \x)}{\ell(u(\x;\tfrac{k_n}{n}) \mid \x)}
			-
			1
		\Bigg|
		\; \diff t \,
		\diff P_{\X}(\x).
	\end{equation*}
	which is bounded by
	\[
		\frac{(1+\xi_U)^2}{\nu_L} 
		\int_{\suppX}
		\int_{t=0}^{\frac{1}{\xi(\x)(\xi(\x) + 1)}}
		\Bigg|
			\frac{\ell \Big( u(\x;\tfrac{k_n}{n}) \bigg( 1 + \frac{s(t,\x)}{u(\x;\tfrac{k_n}{n})} \bigg) \mid \x \Big)}{\ell(u(\x;\tfrac{k_n}{n}) \mid \x)}
			-
			1
		\Bigg|
		\; \diff t \,
		\diff P_{\X}(\x).
	\]
	For $(\x,t) \in \suppX \times \R_+$ such that $t < \tfrac{1}{\xi(\x)(\xi(\x) + 1)}$, we have thanks to Assumption \ref{ass:second-order-RV}
	\begin{multline*}
		\Bigg|
				\frac{\ell \Big( u(\x;\tfrac{k_n}{n}) \bigg( 1 + \frac{s(t,\x)}{u(\x;\tfrac{k_n}{n})} \bigg) \mid \x \Big)}{\ell(u(\x;\tfrac{k_n}{n}) \mid \x)}
				-
				1
		\Bigg|
		\leq
        C \Psi(u(\x;\tfrac{k_n}{n}))
		\frac{1}{\rho(\x)} \lacc 1 - \lp 1 + \frac{s(t,\x)}{u(\x;\tfrac{k_n}{n})} \rp^{\rho(\x)} \racc
		\\ \; + \;
		\operatorname{o} \lp \Psi(u(\x;\tfrac{k_n}{n})) \rp
	\end{multline*}
	as $n \to \infty$. As the map
	\[
		(\x,t) \in \suppX \times \R_+
		\mapsto
        \frac{1}{\rho(\x)} \lacc 1 - \lp 1 + \frac{s(t,\x)}{u(\x;\tfrac{k_n}{n})} \rp^{\rho(\x)} \racc
	\]
	is integrable on $\{(\x,t): t < \tfrac{1}{\xi(\x)(\xi(\x) + 1)}\}$, we get that the bias associated with the first derivative is of the order
	\[
		\operatorname{O} \lp \sup_{\x \in \suppX} \Psi(u(\x;\tfrac{k_n}{n})) \rp, \qquad n \to \infty.
	\]
	
	\paragraph*{Analysis of the term $\boldsymbol{j=2}$.}
	Recall that for $\bphi = (\nu,\xi) \in \reparamspace$ and $z > 0$,
	\[
		\Dot{L}_2(\bphi,z)
        = \lp 1 + \frac{1}{\xi} \rp \frac{(1+2\xi)z}{\nu + \xi(\xi+1) z} - \frac{1}{\xi^2} \log \lp 1 + \frac{\xi(\xi+1)}{\nu}z \rp - \frac{1}{\xi + 1}.
	\]
	Hence if $\F(\x) = (\nu(\x),\xi(\x))$, we get the following bound on \eqref{eq:bias-process} for $j=2$,
	\begin{align*}
		&\int_{\suppX} 
		\lp 1 + \frac{1}{\xi(\x)} \rp \lp 1 + 2\xi(\x) \rp 
		\Bigg|
		\E_{Y|\X=\x} \lc \frac{Y - u(\x;\tfrac{k_n}{n})}{\nu(\x)+\xi(\x)(1+\xi(\x))(Y - u(\x;\tfrac{k_n}{n})} \mid Y > u(\x;\tfrac{k_n}{n}) \rc \\
		&\hspace{7cm} - 
		\E_{Z|\X=\x} \lc \frac{Z}{\nu(\x)+\xi(\x)(1+\xi(\x)) Z} \rc
		\Bigg| \diff P_{\X}(\x) \\
		&\qquad + \int_{\suppX}  \frac{1}{\xi(\x)^2}
		\Bigg|
		\E_{Y|\X=\x} \lc \log \lp 1 + \xi(\x)(1+\xi(\x)) \frac{Y- u(\x;\tfrac{k_n}{n})}{\nu(\x)} \rp  \mid Y > u(\x;\tfrac{k_n}{n})   \rc \\
		&\hspace{5cm} - 
		\E_{Z|\X=\x} \lc \log \lp 1 + \xi(\x)(1+\xi(\x)) \frac{Z}{\nu(\x)} \rp  \rc
		\Bigg| \diff P_{\X}(\x) .
	\end{align*}
	The first term is nearly identical to the one appearing for $j=1$ and is of the order
	\[
		\operatorname{O} \lp \sup_{\x \in \suppX} \Psi(u(\x;\tfrac{k_n}{n})) \rp, \qquad n \to \infty
	\]
	under Assumption \ref{ass:second-order-RV}. For the second term, we compute using \eqref{eq:formula-E} that it equals
	\begin{multline*}
		\int_{\suppX}
		 \frac{1}{\xi(\x)^2}
		\int_{t=0}^{1/\xi(\x)}
		\Bigg|
			\frac{\overline{F}_{Y|\X=\x} \lp u(\x;\tfrac{k_n}{n}) + \frac{\nu(\x)}{\xi(\x)(1+\xi(\x)} (\ee^t - 1) \rp}{\overline{F}_{Y|\X=\x} \lp u(\x;\tfrac{k_n}{n}) \rp}
			\\ -
			\overline{F}_{Z|\X=\x} \lp \frac{\nu(\x)}{\xi(\x)(1+\xi(\x)}(\ee^t - 1) \rp
		\Bigg|
		\; \diff t \,
		\diff P_{\X}(\x).
	\end{multline*}
	Set
	\[
		r(t,\x) = \frac{\nu(\x)}{\xi(\x)(1+\xi(\x)} (\ee^t - 1), \qquad \forall (t,\x) \in \R_+ \times \suppX.
	\]
	The previous integral may be bounded using similar arguments as in the case $j=1$ by
	\[
		\frac{1}{\xi_L^2} 
		\int_{\suppX} 	\int_{t=0}^{\infty}
		\Bigg|
		\frac{\ell \Big( u(\x;\tfrac{k_n}{n}) \bigg( 1 + \frac{r(t,\x)}{u(\x;\tfrac{k_n}{n})} \bigg) \mid \x \Big)}{\ell(u(\x;\tfrac{k_n}{n}) \mid \x)}
		-
		1
		\Bigg|
		\; \diff t \,
		\diff P_{\X}(\x).
	\]
	Using Assumption \ref{ass:second-order-RV}, we get
	\begin{multline*}
		\Bigg|
		\frac{\ell \Big( u(\x;\tfrac{k_n}{n}) \bigg( 1 + \frac{s(t,\x)}{u(\x;\tfrac{k_n}{n})} \bigg) \mid \x \Big)}{\ell(u(\x;\tfrac{k_n}{n}) \mid \x)}
		-
		1
		\Bigg| 
		\leq
        C \Psi(u(\x;\tfrac{k_n}{n}))
		\frac{1}{\rho(\x)} \lacc 1 - \lp 1 + \frac{r(t,\x)}{u(\x;\tfrac{k_n}{n})} \rp^{\rho(\x)} \racc \\
		\; + \;
		\operatorname{o} \lp \Psi(u(\x;\tfrac{k_n}{n})) \rp
	\end{multline*}
	as $n \to \infty$. Similarly as for the previous case, we get that the bias associated with the second derivative is of the order
	\[
		\operatorname{O} \lp \sup_{\x \in \suppX} \Psi(u(\x;\tfrac{k_n}{n})) \rp, \qquad n \to \infty.
	\]
	This concludes the proof.
\end{proof}

\section{Additional figures}
\label{app:figures}

\begin{figure}[h]
    \centering
    \includegraphics[width=0.3\linewidth]{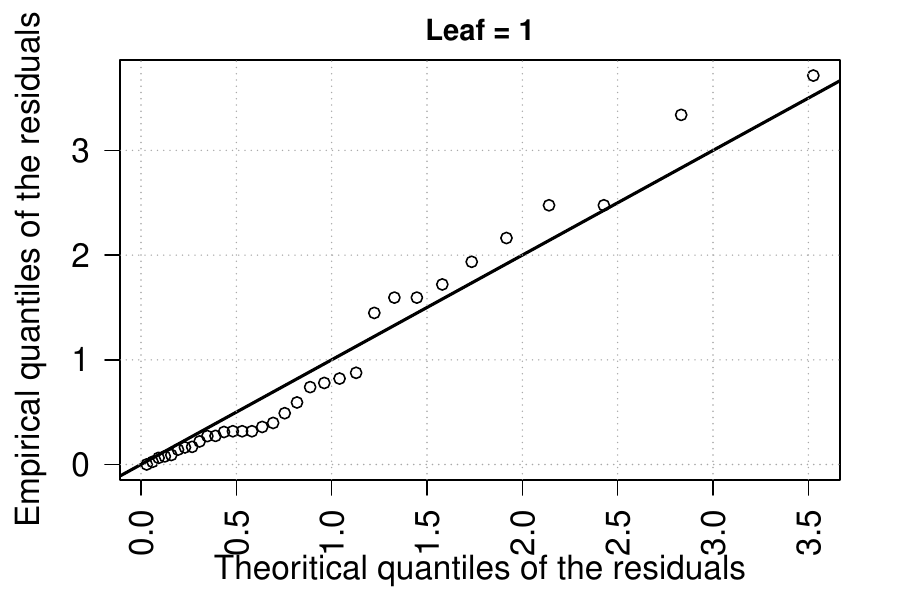}
    \includegraphics[width=0.3\linewidth]{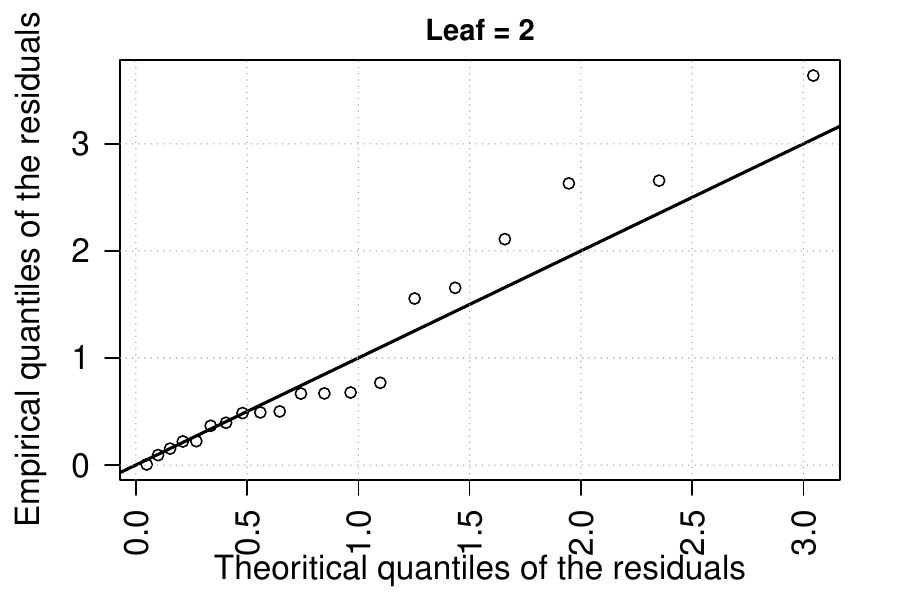}
    \includegraphics[width=0.3\linewidth]{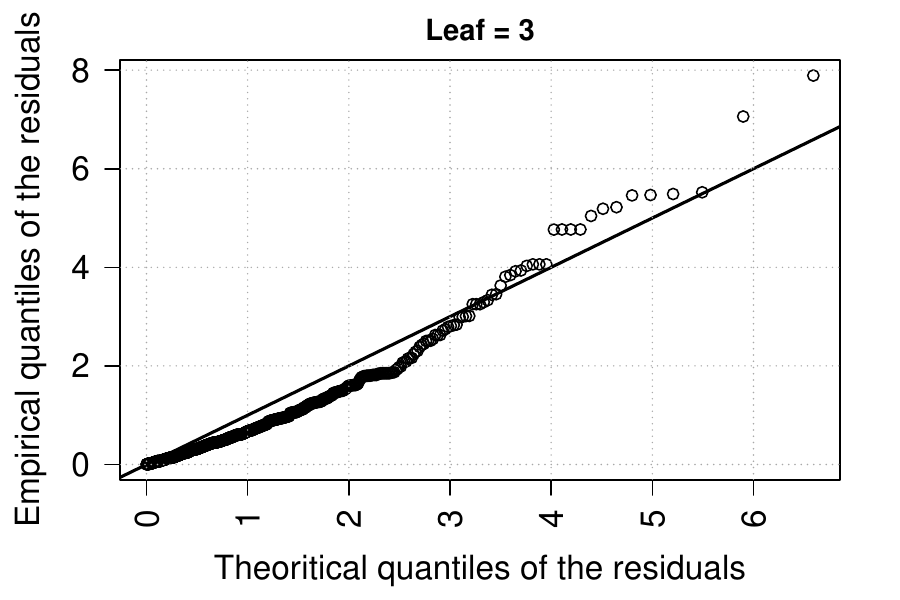}
    \includegraphics[width=0.3\linewidth]{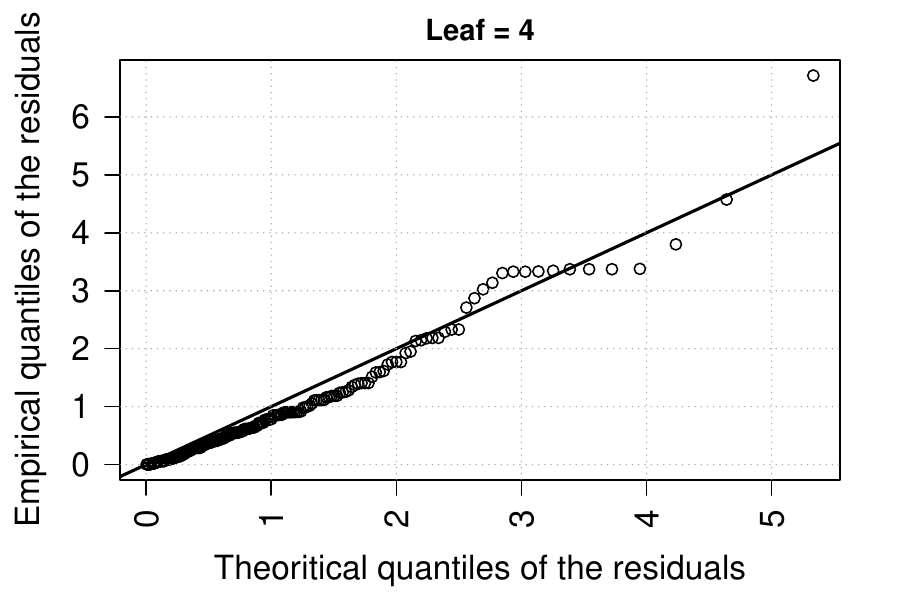}
    \includegraphics[width=0.3\linewidth]{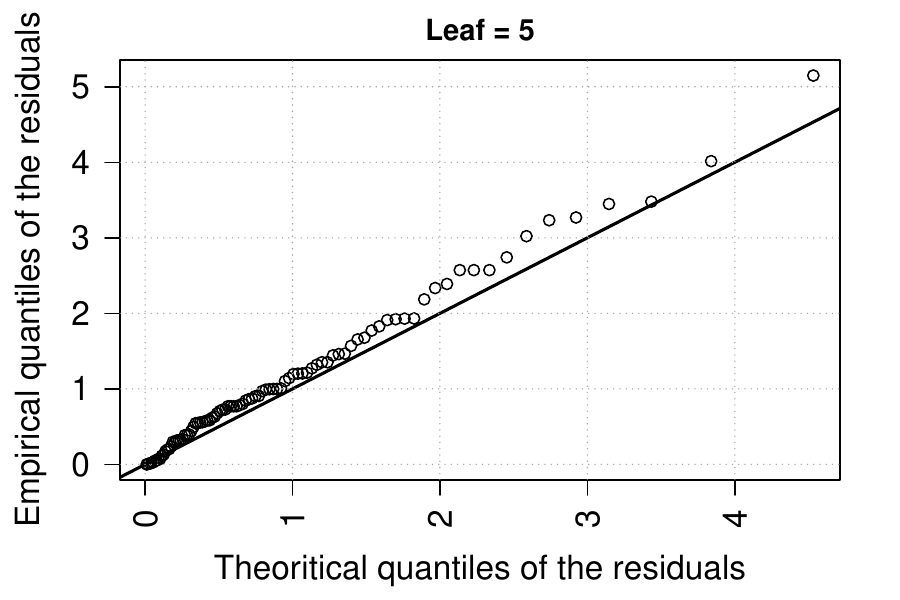}
    \includegraphics[width=0.3\linewidth]{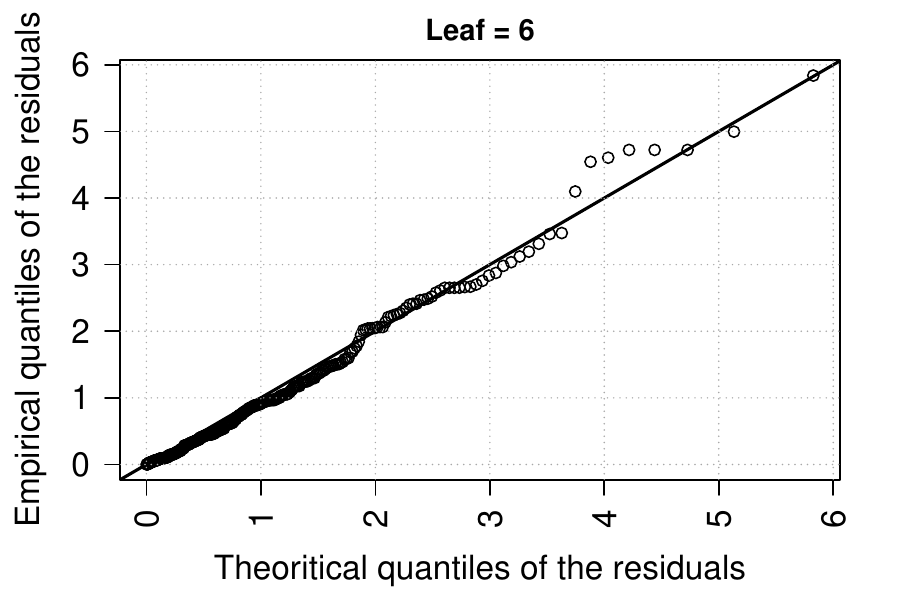}
    \caption{Quantile-quantile plot of the GP distribution fit in each leaf provided by the GPtree algorithm in order to choose a locally constant threshold in Section \ref{subsec:application}.}
    \label{fig:qq-leafs}
\end{figure}

\end{document}